\documentclass{article}

\usepackage{microtype}
\usepackage{graphicx}
\usepackage{subfigure}
\usepackage{booktabs} % for professional tables

\usepackage{hyperref}

\usepackage[accepted]{icml2020}

\usepackage{amsmath}
\usepackage{amsfonts}
\usepackage{amsthm}
\usepackage{appendix}
\usepackage{dsfont}
\usepackage{bm}
\usepackage{multirow}
\icmltitlerunning{Poisson Learning}

\newcommand{\uAve}{\bar{y}_{\mathrm{w}}}
\newcommand{\cons}{\bar{y}}
\newcommand{\bcons}{\bar{\mathbf{y}}}
\newcommand{\dd}{\mathrm{d}}
\newcommand{\bbR}{\mathbb{R}}
\newcommand{\cE}{\mathcal{E}}
\newcommand{\Ninner}{N_{\mathrm{inner}}}
\newcommand{\Nouter}{N_{\mathrm{outer}}}

\newcommand{\GL}{\mathrm{GL}}

\definecolor{mygreen}{rgb}{0.1,0.75,0.2}

\def\lp{\left(}
\def\rp{\right)}
\def\la{\left|}
\def\ra{\right|}
\def\lda{\left\|}
\def\rda{\right\|}
\def\lb{\left\{}

\def\rd{\right.}

\definecolor{darkred}{rgb}{0.6,0.1,0.1}
\definecolor{darkblue}{rgb}{0.1,0.1,0.6}

\newcommand{\diag}{\text{diag}}
\newcommand{\proj}{\text{Proj}}
\newcommand{\bproj}{\textbf{Proj}}
\newcommand{\one}{\mathds{1}}
\newcommand{\zero}{\mathbf{0}}
\newcommand{\dist}{\text{dist}}
\DeclareMathOperator*{\argmax}{arg\,max}

\renewcommand{\div}{\textnormal{div}\,}

\newcommand{\M}{{\mathcal M}}
\renewcommand{\L}{{\mathcal L}}

\newcommand{\bp}{{\mathbf p}}
\newcommand{\bs}{{\mathbf s}}

\renewcommand{\P}{{\mathbb P}}
\newcommand{\E}{{\mathbb E}}

\renewcommand{\O}{\Gamma}
\renewcommand{\bar}[1]{{\overline{#1}}}

\newcommand{\e}{{\mathbf e}}
\newcommand{\bW}{{\mathbf W}}

\newcommand{\bL}{{\mathbf L}}

\newcommand{\bD}{{\mathbf D}}
\newcommand{\bF}{{\mathbf F}}
\newcommand{\bB}{{\mathbf B}}

\newcommand{\bU}{{\mathbf U}}
\newcommand{\bb}{{\mathbf b}}

\newcommand{\R}{\mathbb{R}}

\newcommand{\eps}{\varepsilon}
\renewcommand{\hat}[1]{\widehat{#1}}
\renewcommand{\tilde}[1]{\widetilde{#1}}
\renewcommand{\phi}{\varphi}

\newcommand{\Vol}{\mathrm{Vol}}
\newcommand{\lpxaz}{\ell^p_{a,0}(X)}
\newcommand{\lx}{\ell^2(X)}
\newcommand{\lonex}{\ell^1(X)}

\newcommand{\lxz}{\ell^2_0(X)}
\newcommand{\lxs}{\ell^2(X^2)}
\newcommand{\lpx}{\ell^p(X)}
\newcommand{\lqx}{\ell^q(X)}
\newcommand{\lpxz}{\ell^p_0(X)}
\newcommand{\lpxs}{\ell^p(X^2)}

\newtheorem{theorem}{Theorem}
\newtheorem{lemma}[theorem]{Lemma}

\newtheorem{proposition}[theorem]{Proposition}
\theoremstyle{definition}

\newtheorem{remark}[theorem]{Remark}

\newtheorem{conjecture}[theorem]{Conjecture}
\numberwithin{theorem}{section}
\numberwithin{equation}{section}
\begin{document}

\twocolumn[
\icmltitle{Poisson Learning: Graph Based Semi-Supervised Learning \texorpdfstring{\\}{} At Very Low Label Rates}



\icmlsetsymbol{equal}{*}

\begin{icmlauthorlist}
\icmlauthor{Jeff Calder}{min}
\icmlauthor{Brendan Cook}{min}
\icmlauthor{Matthew Thorpe}{man}
\icmlauthor{Dejan Slep\v{c}ev}{cmu}
\end{icmlauthorlist}

\icmlaffiliation{min}{School of Mathematics, University of Minnesota, Minneapolis, USA.}
\icmlaffiliation{man}{Department of Mathematics, University of Manchester, Manchester, UK.}
\icmlaffiliation{cmu}{Department of Mathematical Sciences, Carnegie Mellon University, Pittsburgh, USA.}

\icmlcorrespondingauthor{Jeff Calder}{jwcalder@umn.edu}

\icmlkeywords{Semi-Supervised Learning, Graph Laplacian, Poisson Learning, Laplacian Learning, MBO Scheme}

\vskip 0.3in
]



\begin{NoHyper}
\printAffiliationsAndNotice{}  % leave blank if no need to mention equal contribution
\end{NoHyper}
\begin{abstract}
We propose a new framework, called \emph{Poisson learning}, for graph based semi-supervised learning at very low label rates. Poisson learning is motivated by the need to address the degeneracy of Laplacian semi-supervised learning in this regime. The method replaces the assignment of label values at training points with the placement of sources and sinks, and solves the resulting Poisson equation on the graph. The outcomes are  provably  more stable and informative than those of Laplacian learning. Poisson learning is efficient and simple to implement, and we present numerical experiments showing the method is superior to other recent approaches to semi-supervised learning at low label rates on MNIST, FashionMNIST, and Cifar-10. We also propose a graph-cut enhancement of Poisson learning, called \emph{Poisson MBO}, that gives higher accuracy and can incorporate prior knowledge of relative class sizes.
%We propose a new framework, called \emph{Poisson learning}, for graph based semi-supervised learning at very low label rates. Poisson learning is motivated by the need to address the degeneracy of Laplacian semi-supervised learning at very low label rates. The method replaces the assignment of label values at training points with the placement of sources and sinks, and solves the resulting Poisson equation on the graph. The outcomes are provably more stable and informative than those of Laplacian learning. Poisson learning is fast and efficient to implement, and we present numerical experiments showing the method is superior to other recent approaches to semi-supervised learning at low label rates on the MNIST, FashionMNIST, and Cifar-10 datasets. We also propose a graph-cut version of Poisson learning, called \emph{Poisson MBO}, that gives higher accuracy and can incorporate prior knowledge of relative class sizes.
\end{abstract}

\section{Introduction} \label{sec:intro}
Semi-supervised learning uses both labeled and unlabeled data in learning tasks. For problems where very few labels are available, geometric or topological structure in unlabeled data can be used to greatly improve the performance of classification, regression, or clustering algorithms. One of the most widely used methods in graph-based semi-supervised learning is Laplace learning, originally proposed in \cite{zhu2003semi}, which seeks a graph harmonic function that extends the labels. Laplace learning, and variants thereof, have been widely applied in semi-supervised learning~\cite{zhou2005learning,zhou2004learning,zhou2004ranking,ando2007learning} and manifold ranking~\cite{he2004manifold,yang2013saliency,xu2011efficient}, among many other problems. 

This paper is concerned with graph-based semi-supervised learning at very low label rates. In this setting, it has been observed that Laplace learning can give very poor classification results~\cite{nadler2009semi,el2016asymptotic}.  The poor results are often attributed to the fact that the solutions develop localized spikes near the labeled points and are almost constant far from the labeled points. In particular, label values are not propagated well by the Laplacian learning approach. To address this issue, recent work has suggested to consider $p$-Laplace learning~\cite{el2016asymptotic}.  Several works have rigorously studied $p$-Laplace regularization with few labels~\cite{slepcev2019analysis,calder2018game,calder2017consistency}, and recent numerical results show that $p>2$ is superior to Laplace learning at low label rates~\cite{flores2019algorithms}. The case of $p=\infty$ is called Lipschitz learning \cite{kyng2015algorithms}, which seeks the absolutely minimal Lipschitz extension of the training data.  Other methods to address low label rate problems include  higher order Laplacian regularization \cite{zhou2011semi} and spectral cutoffs \cite{Belkin2002UsingMS}.

While $p$-Laplace learning improves upon Laplace learning, the method is more computationally burdensome than Laplace learning, since the optimality conditions are nonlinear. Other recent approaches have aimed to re-weight the graph more heavily near labels, in order to give them wider influence when the labeling rate is very low. One way to re-weight the graph is the Weighted Nonlocal Laplacian (WNLL) \cite{shi2017weighted}, which amplifies the weights of edges directly connected to labeled nodes. The WNLL achieves better results at moderately low label rates, but still performs poorly at very low label rates \cite{flores2019algorithms}. To address this, \cite{calder2018properly} proposed the Properly Weighted Laplacian, which re-weights the graph in a way that is well-posed at arbitrarily low label rates. 

Much of the recent work on  low label rate problems has focused on formulating and implementing new learning approaches  that are well-posed with few labels.  The exact nature of the degeneracy in Laplace learning, and the question of how the tails of the spikes propagate label information, has not been studied and is still poorly understood. For some problems, the performance of Laplacian learning is good \cite{zhu2003semi},   while for other problems it is catastrophic, yielding very poor classification results similar to random guessing \cite{shi2017weighted,flores2019algorithms}. 

In this paper, we carefully analyze Laplace learning at very low label rates, and we discover that \emph{nearly all} of the degeneracy of Laplace learning is due to a large constant bias in the solution of the Laplace equation that is present only at low label rates. 
In order to overcome this problem we introduce a new algorithm, we call \emph{Poisson learning}, that gives very good classification performance down to extremely low label rates. We give a random walk interpretation of Poisson learning that shows how the method uses information from the random walkers only \emph{before} they reach the mixing time of the random walk and forget their initial condition.  We also propose a graph-cut enhancement
of Poisson learning, called \emph{Poisson MBO}, that can incorporate knowledge about class sizes, and further improves the class-label accuracy. 

The rest of the paper is organized as follows.  In Section \ref{sec:analysis} we first briefly  introduce Poisson learning, and then provide a detailed motivation for the algorithm and a theoretical analysis from both the variational and random walk perspectives. The Poisson MBO approach is presented in Section \ref{sec:poissonMBO}. In Section \ref{sec:algorithms} we present the step-by-step algorithms and discuss implementation details for the Poisson and Poisson MBO algorithms. In Section \ref{sec:num} we present numerical experiments with semi-supervised classification on the MNIST, FashionMNIST, and Cifar-10 datasets. The proofs of the results are available in the supplementary materials.
%all results are postponed to the supplementary materials.

\section{Poisson learning} \label{sec:analysis}

Let $X=\{x_1,x_2,\dots,x_n\}$ denote the vertices of a graph with edge weights $w_{ij}\geq 0$ between $x_i$ and $x_j$. We assume the graph is symmetric, so $w_{ij}=w_{ji}$. We define the degree $d_i=\sum_{j=1}^n w_{ij}$. For a multi-class classification problem with $k$ classes, we let the standard basis vector $\e_i\in \R^k$ represent the $i^{\rm th}$ class. We assume the first $m$ vertices $x_1,x_2,\dots,x_m$ are given labels $y_1,y_2,\dots,y_m\in \{\e_1,\e_2,\dots,\e_k\}$, where $m < n$. The task of graph-based semi-supervised learning is to extend the labels to the rest of the vertices $x_{m+1},x_{m+2},\dots,x_{n}$. 

The well-known Laplace learning algorithm \cite{zhu2003semi} extends the labels by solving the problem
\begin{equation}\label{eq:bv}
\left.\begin{aligned}
\L u(x_i) &= 0,&&\text{if }m+1 \leq i \leq n,\\ 
u(x_i) &=y_i,&&\text{if }1 \leq i \leq m,
\end{aligned}\right\}
\end{equation}
where $\L$ is the unnormalized graph Laplacian given by 
\[\L u(x_i) = \sum_{j=1}^n w_{ij}(u(x_i) - u(x_j)).\]
Here, $u:X\to \R^k$ and we write the components of $u$ as $u(x_i)=(u_1(x_i),u_2(x_i),\dots,u_k(x_i))$. The label decision for vertex $x_i$ is determined by the largest component of $u(x_i)$
\begin{equation}\label{eq:labeldec}
\ell(x_i) = \argmax_{j\in \{1,\dots,k\}} \{u_j(x)\}.
\end{equation}
We note that Laplace learning is also called \emph{label propagation (LP)} \cite{zhu2005semi}, since the Laplace equation \eqref{eq:bv} can be solved by repeatedly replacing $u(x_i)$ with the weighted average of its neighbors, which can be viewed as dynamically propagating labels.

At very low label rates, we propose to replace the problem \eqref{eq:bv} by \emph{Poisson learning}: Let $\cons=\tfrac{1}{m}\sum_{j=1}^m y_j$ be the average label vector and let  $\delta_{ij}=1$ if $i=j$ and $\delta_{ij}=0$ if $i\neq j$. One
 computes the solution of the Poisson equation
\begin{equation}\label{eq:PoissonL}
\L u(x_i) = \sum_{j=1}^m(y_j-\cons)\delta_{ij} \ \ \ \text{for }i=1,\dots,n
\end{equation}
satisfying $\sum_{i=1}^n d_i u(x_i) = 0$.  While the label decision can be taken to be the same as \eqref{eq:labeldec}, it is also simple to account for unbalanced classes or training data with the modified label decision
\begin{equation}\label{eq:modified_labeldec}
\ell(x_i) = \argmax_{j\in \{1,\dots,k\}} \{s_j u_j(x)\},
\end{equation}
where $s_j = b_j/(\bar{y}\cdot \e_j)$ and $b_j$ is the fraction of data belonging to class $j$. We explain this label decision in Remark \ref{rem:linearity}. The Poisson equation \eqref{eq:PoissonL} can be solved efficiently with a simple iteration given in Algorithm \ref{alg:poisson}.

Technically speaking, in Laplace learning, the labels are imposed as boundary conditions in a Laplace equation, while in Poisson learning, the labels appears as a source term in a graph Poisson equation. In the sections below, we explain why Poisson learning is a good idea for problems with very few labels.  In particular, we give random walk and variational interpretations of Poisson learning, and we illustrate how Poisson learning arises as the low label rate limit of Laplace learning.

\subsection{Random walk interpretation} \label{sec:randomwalk}

We present a random walk interpretation of Poisson learning and compare to the random
walk interpretation of Laplace learning to explain its poor performance at low label rates. We note that Laplace learning works very well in practice for semi-supervised learning problems with a moderate amount of labeled data. For example, on the MNIST dataset we obtained around 95\% accuracy at 16 labels per class (0.23\% label rate). However, at very low label rates the performance is poor. At 1 label per class, we find the average performance is around 16\% accuracy. This phenomenon has been observed in other works recently \cite{nadler2009semi,el2016asymptotic}. However, a clear understanding of the issues with Laplace learning at low label rates was lacking. The clearest understanding of this phenomenon comes from the random walk interpretation, and this leads directly to the foundations for Poisson learning.

Let $x\in X$ and let $X_0^x,X_1^x,X_2^x,\dots$ be a random walk on $X$ starting at $X_0^x =x$ with transition probabilities 
\[ \P(X^x_k = x_j \, | \, X^x_{k-1} = x_i) = d_i^{-1}w_{ij}. \]
Let $u$ be the solution of the \emph{Laplace} learning problem \eqref{eq:bv}. Define the stopping time to be the first time the walk hits a label, that is
\[\tau = \inf\{k\geq 0\, : \, X^x_k \in \{x_1,x_2,\dots,x_m\}\}.\]
Let $i_\tau\leq m$ denote the index of the point $X^x_\tau$, so $X^x_\tau = x_{i_\tau}$. Then, by Doob's optimal stopping theorem, we have
\begin{equation}\label{eq:urep}
u(x) = \E[y_{i_\tau}]. 
\end{equation}
This gives the standard representation formula for the solution of  Laplace learning \eqref{eq:bv}.  The interpretation is that we release a random walker from $x$ and let it walk until it hits a labeled vertex and then record that label. We average over many random walkers to get the value of $u(x)$.  

When there are insufficiently many labels, the stopping time $\tau$ is so large that we have passed the \emph{mixing time} of the random walk, and the distribution of $X^x_\tau$ is very close to the invariant distribution of the random walk $\pi(x_i) = d_i/\sum_i d_i$. In this case,  \eqref{eq:urep} gives that
\begin{equation}\label{eq:rwexplain}
u(x) \approx \frac{\sum_{i=1}^m d_i y_i}{\sum_{j=1}^md_j} =: \uAve.
\end{equation}
Of course, the function $u$ is not exactly constant, and instead it is approximately constant with sharp spikes at the labeled vertices; see Figure \ref{fig:spikes}. Previous work has proved rigorously that Laplace learning degenerates in this way at low label rates \cite{slepcev2019analysis,calder2018game}.
\begin{figure}[t!]
\vskip 0.2in
\begin{center}
\centerline{\includegraphics[trim=550 200 500 180, clip=true,width=0.75\columnwidth]{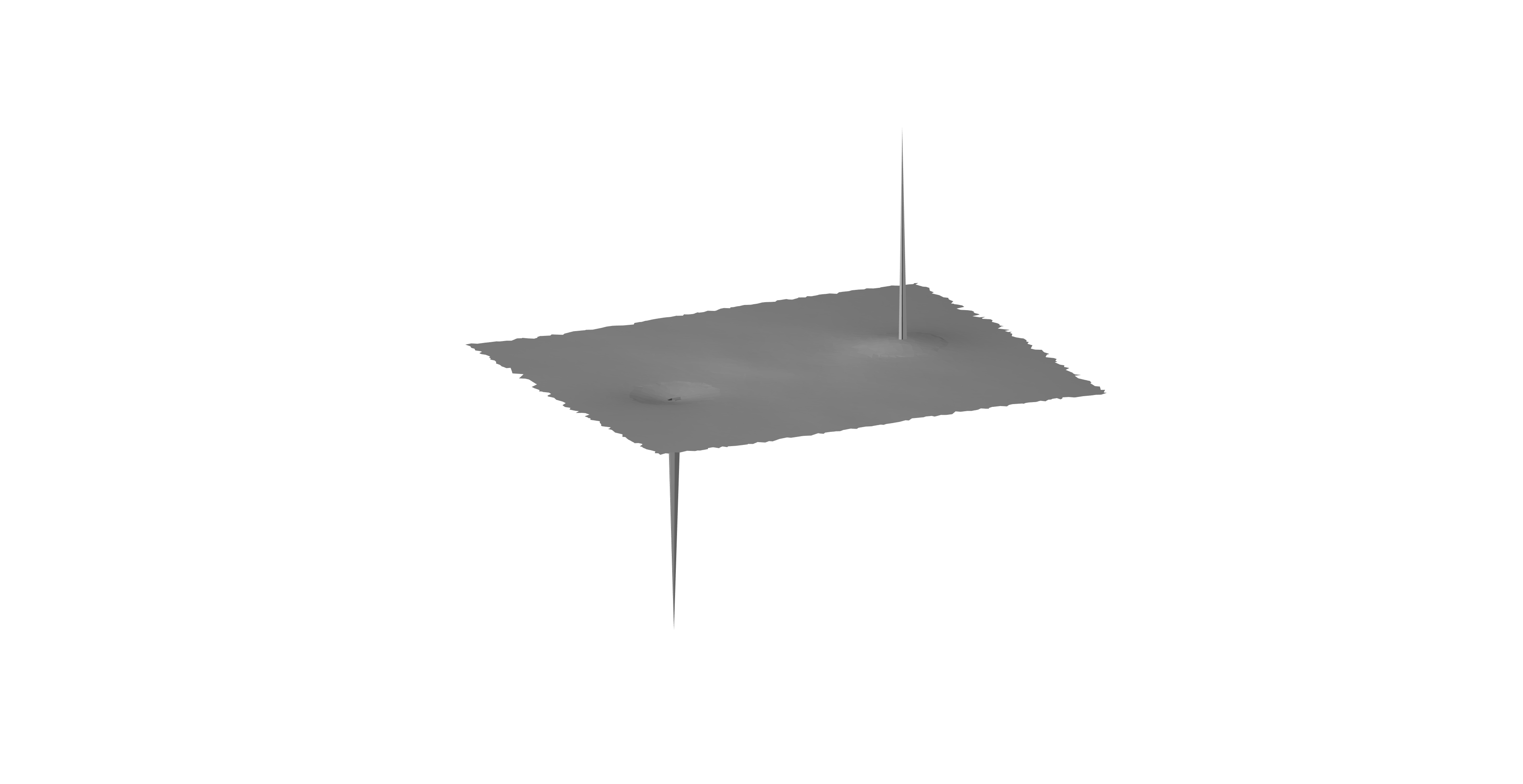}}
\caption{Demonstration of spikes in Laplace learning. The graph consists of $n=10^4$ independent uniform random variables on $[0,1]^2$ and two points are given labels of $0$ and $1$. Most values of the solution $u$ of Laplace learning are very close to $0.5$.}
\label{fig:spikes}
\end{center}
\vskip -0.2in
\end{figure}

It is important to note that the constant vector $\uAve$ depends only on the degrees of the \emph{labeled} nodes in the graph, which is very sensitive to local graph structure. When Laplace learning returns a nearly constant label function, it can be catastrophic for classification, since most datapoints are assigned the same label. This explains the 16\% accuracy in the MNIST experiment described above. 

In contrast, the Poisson equation \eqref{eq:PoissonL} has a random walk interpretation that involves the Green's function for a random walk, and is in some sense dual to Laplace learning. The source term on the right hand side of \eqref{eq:PoissonL} represents random walkers being released from labeled points and exploring the graph, while carrying their label information with them. For $T>0$ let us define
\[w_T(x_i) = \E\left[\sum_{k=0}^T \sum_{j=1}^m y_j\one_{\{X^{x_j}_k=x_i\}}\right].\]
Essentially, each time the random walk starting from $x_j$, denoted $X^{x_j}_k$, visits $x_i$ we record the label $y_j$ and sum this over all visits from $0 \leq k \leq T$. For short time $T$, this quantity is meaningful, but since the random walk is recurrent (it is a finite graph), $w_T(x_i)\to \infty$ as $T\to \infty$. If we normalize by $1/T$, we still have the same issue as with Laplace learning, where we are only measuring the invariant distribution. Indeed, we note that
\[w_T(x_i) = \sum_{k=0}^T \sum_{j=1}^m y_j\P(X^{x_j}_k=x_i)\]
\vspace{-2mm}
\[\text{and } \ \  \lim_{k\to \infty}\P(X^{x_j}_k=x_i) = \frac{d_i}{\sum_{i=1}^n d_i}.\]
Therefore, when $k$ is large we have
\[\sum_{j=1}^m y_j\P(X^{x_j}_k=x_i) \approx \frac{d_i}{\sum_{i=1}^nd_i}\sum_{j=1}^my_j,\]
and so the tail of the sum defining $w_T(x_i)$ is recording a blind average of labels. 

The discussion above suggests that we should subtract off this average tail behavior from $w_T$, so that we only record the \emph{short-time} behavior of the random walk, before the mixing time is reached. We also normalize by $d_i$, which leads us to define 
\begin{equation}\label{eq:PoissonRandomWalk}
u_T(x_i) = \E\left[\sum_{k=0}^T \frac{1}{d_i}\sum_{j=1}^m(y_j-\cons)\one_{\{X^{x_j}_k=x_i\}}\right], 
\end{equation}
where $\cons = \frac{1}{m}\sum_{j=1}^m y_j$. It turns out that as $T\to \infty$, the function $u_T(x_i)$ converges to the solution of \eqref{eq:PoissonL}.
\begin{theorem}\label{thm:PoissonRandomWalk}
For every $T\geq 0$ we have
\[ u_{T+1}(x_i)  = u_{T}(x_i) + d_i^{-1}\hspace{-1.5mm}\left(\sum_{j=1}^m (y_j-\cons)\delta_{ij}-\L u_{T}(x_i)\right). \]
If the graph $G$ is connected and the Markov chain induced by the random walk is aperiodic, then $u_T \to u$ as $T\to \infty$, where $u$ is the unique solution of the Poisson equation \eqref{eq:PoissonL} satisfying $\sum_{i=1}^n d_iu(x_i) = 0$.
\end{theorem}
Theorem \ref{thm:PoissonRandomWalk} gives the foundation for Poisson learning through the random walk perspective, and in fact, it also gives a numerical method for computing the solution (see Algorithm \ref{alg:poisson}). 

\begin{remark}
The representation formula \eqref{eq:PoissonRandomWalk} for the solution of Poisson learning \eqref{eq:PoissonL} shows that the solution $u$ is a \emph{linear} function of the label vectors $y_1,\dots,y_m$. That is, for any $A\in \R^{k\times k}$, the solution $u_A:X\to \R^k$ of
\[\L u_A(x_i) = \sum_{j=1}^m(Ay_j-A\cons)\delta_{ij} \ \ \ \text{for }i=1,\dots,n\]
satisfying $\sum_{i=1}^n d_i u_A(x_i) =0 $ is exactly $u_A = Au$, where $u$ is the solution of \eqref{eq:PoissonL}. This shows that any reweighting of the point sources, by which we mean $y_j \mapsto A y_j$, is equivalent to reweighting the solution by $u \mapsto Au$.

If we set $A = \diag(s_1,\dots,s_k)$, then $Au$ corresponds to multiplying the point sources for class $i$ by the weight $s_i$. We can use this reweighting to account for unbalanced classes, or a discrepancy between the balancing of training and testing data, in the following way. Let $n_j$ be the number of training examples from class $j$, and let $b_j$ denote the true fraction of data in class $j$. We can choose $s_j$ so that $n_j s_j = b_j$ to ensure that the mass of the point sources for each class, weighted by $s_j$, is proportional to the true fraction of data in that class. Since $n_j$ is proportional to $\bar{y}\cdot \e_j$, this explains our modified label decision \eqref{eq:modified_labeldec}. 
\label{rem:linearity}
\end{remark}

\subsection{Variational interpretation} \label{sec:variational}

We can also interpret Poisson learning \eqref{eq:PoissonL} as a gradient regularized variational problem. Before proceeding, we briefly review some facts about calculus on graphs. Let $\lx$ denote the space of functions $u:X\to \R^k$ equipped with the inner product
\[ (u,v)_{\lx} = \sum_{i=1}^nu(x_i)\cdot v(x_i). \]
This induces a norm $\|u\|^2_{\lx} = (u,u)_{\lx}$. We also define  the space of mean-zero functions
\[ \lxz = \Big\{u \in \lx \, : \, \sum_{i=1}^n d_i u(x_i)=0\Big\}\]
We define a \emph{vector field} on the graph to be an \emph{antisymmetric} function $V:X^2\to \R^k$ (i.e., $V(x_i,x_j) = -V(x_j,x_i)$).  The \emph{gradient} of a function $u\in \lx$, denoted for simplicity as $\nabla u$, is defined to be the vector field
\[ \nabla u(x_i,x_j) = u(x_j) - u(x_i). \]
The inner product between vector fields $V$ and $W$ is
\[ (V,W)_{\lxs} = \frac{1}{2}\sum_{i,j=1}^nw_{ij}V(x_i,x_j)\cdot W(x_i,x_j). \]
and the norm of $V$ is $\|V\|^2_{\lxs} = (V,V)_{\lxs}$. 

We now consider the variational problem
\begin{equation}\label{eq:PoissonVar2}
\min_{u\in \lxz} \hspace{-3pt} \bigg\{ \frac{1}{2}\|\nabla u\|_{\lxs}^2 \hspace{-1pt} - \hspace{-1pt}  \sum_{j=1}^m (y_j - \cons)\hspace{-1.5pt}\cdot\hspace{-1.5pt}u(x_j) \bigg\}. \hspace{-3pt}
\end{equation}
 The following theorem makes the connection between Poisson learning and the variational problem \eqref{eq:PoissonVar2}.
\begin{theorem}\label{thm:exist2}
Assume $G$ is connected. Then there exists a unique minimizer $u\in \lxz$ of \eqref{eq:PoissonVar2}, and $u$ satisfies the Poisson learning equation \eqref{eq:PoissonL}.
\end{theorem}
Theorem \ref{thm:exist2} shows that the Poisson learning equation \eqref{eq:PoissonL} arises as the necessary conditions for a gradient regularized variational problem with an affine loss function. We contrast this with the solution of Laplace learning \eqref{eq:bv}, which is the minimizer of the variational problem
\begin{equation}\label{eq:Laplace}
\min_{u\in \lx}\left\{\|\nabla u\|_{\lxs}^2 \, : \, u(x_i)=y_i, \, 1\leq i\leq m\right\}.
\end{equation}
Thus, while both Laplace and Poisson learning are gradient regularized variational problems, the key difference is how each algorithm handles the labeled data; Laplace learning enforces hard label constraints while Poisson learning adds an affine loss function to the energy. Of course, many variants of Laplace learning have been proposed with various types of soft label constraints in place of hard constraints. These variations perform similarly poorly to Laplace learning at low label rates, and the key feature of Poisson learning is the \emph{affine} loss function that can be easily centered. 

We note that it would be natural to add a weight $\mu>0$ to one of the terms in \eqref{eq:PoissonVar2} to trade-off the importance of the two terms in the variational problem. However, as we show in the supplementary material (see Lemma A.3), this weight would have no effect on the final label decision. We also mention that the variational problem \eqref{eq:PoissonVar2} has a natural $\ell^p$ generalization that we also explore in the supplementary material (Section A.2).

\subsection{Laplace learning at low label rates} \label{sec:motivation}

Finally, to further motivate Poisson learning, we connect Poisson learning with the limit of Laplace learning at very low label rates. At low label rates, the solution of Laplace learning \eqref{eq:bv} concentrates around a constant $\uAve$, for which we gave an interpretation of via random walks in Section \ref{sec:randomwalk}. Near labeled nodes, $u$ has sharp spikes (recall Figure~\ref{fig:spikes}) in order to attain the labels. From the variational perspective, the constant function has zero cost, and the cost of spikes is very small, so this configuration is less expensive than continuously attaining the labels $u(x_i)=y_i$. 

A natural question concerns whether the spikes in Laplace learning contain useful information, or whether they are too localized and do not propagate information well. To test this, we changed the label decision \eqref{eq:labeldec} in the MNIST experiment described in Section \ref{sec:randomwalk} to subtract off the tail constant $\uAve$ identified in \eqref{eq:rwexplain}
\[ \ell(x_i) = \argmax_{j\in \{1,\dots,k\}} \{u_j(x) - \uAve \cdot \e_j\}. \]
where $\uAve$ is defined in \eqref{eq:rwexplain}. Thus, we are centering the function $u$ at zero. At 1 label per class (10 labeled images total), the accuracy improved from 16\% to 85.9\%! Hence, the difference $u-\uAve$ contains enough information to make informed classification decisions, and therefore the spikes contain useful information.

This indicates that much of the poor performance of Laplace learning can be explained by a large shift bias that occurs only at very low label rates. Fixing this seems as simple as applying the appropriate shift before making a decision on a label, but this does not lead to a well-grounded method, since the shifted function $u - \uAve$ is exactly the graph-harmonic extension of the shifted labels $y_j-\uAve$. Why should we have to use a harmonic extension of the \emph{wrong labels} in order to achieve a better result? On the other hand, Poisson learning, which we introduced above, provides an intuitive and  well-grounded way of fixing the shift bias in Laplace learning.

To see the connection to Poisson learning, let us assume the solution $u$ of the Laplace learning equation \eqref{eq:bv} is nearly equal to the constant vector $\uAve\in \R^k$ from \eqref{eq:rwexplain} at all unlabeled points $x_{m+1},\dots,x_n$. For any \emph{labeled} node $x_i$ with $i=1,\dots,m$ we can compute (assuming $w_{ij} = 0$ for all $j\in \{1,\dots,m\}$) that
\begin{align*}
\L u(x_i)&=\sum_{j=1}^nw_{ij}(u(x_i)-u(x_j))\\
&\approx \sum_{j=m+1}^nw_{ij}(y_i-\uAve)= d_i(y_i-\uAve).
\end{align*}
Since $\L u(x_i)=0$ for $i=m+1,\dots,n$, we find that $u$ approximately satisfies the Poisson equation
\begin{equation}\label{eq:Poissonw}
\L u(x_i) = \sum_{j=1}^md_j(y_j-\uAve)\delta_{ij} \ \ \ \text{for }i=1,\dots,n.
\end{equation}
This gives a connection, at a heuristic level, been Laplace equations with hard constraints, and Poisson equations with point sources, for problems with very low label rates. We note that since constant functions are in the kernel of $\L$, $u-\uAve$ also satisfies \eqref{eq:Poissonw}. We also note that the labels, and the constant $\uAve$, are weighted by the degree $d_j$, which does not appear in our Poisson learning equation \eqref{eq:PoissonL}. We have found that both models give good results, but that \eqref{eq:PoissonL} works slightly better, which is likely due to the rigorous foundation of \eqref{eq:PoissonL} via random walks.

\subsection{The Poisson MBO algorithm} \label{sec:poissonMBO}

Poisson learning provides a robust method for propagating label information that is stable at very low label rates. After applying Poisson learning to propagate labels, we propose a graph-cut method to incrementally adjust the decision boundary so as to improve the label accuracy and account for prior knowledge of class sizes. The graph-cut method we propose is to apply several steps of gradient descent on the graph-cut problem
\begin{equation}\label{eq:graphcut}
\min_{\substack{u:X\to S_k \\ (u)_X=b}}\hspace{-1mm}\bigg\{ \frac{1}{2}\|\nabla u\|^2_{\lxs} -\mu \sum_{j=1}^m (y_j - \cons)\hspace{-0.5mm}\cdot\hspace{-0.5mm} u(x_j)  \bigg\},
\end{equation}
where $S_k = \{\e_1,\e_2,\dots,\e_k\}$, $b\in \R^k$ is given, and $(u)_X := \frac{1}{n}\sum_{i=1}^n u(x_i)$.  Since we are restricting $u(x_i)\in S_k$, the term $\tfrac{1}{2}\|\nabla u\|^2_{\lxs}$ is exactly the graph-cut energy of the classification given by $u$. Likewise, the components of the average $(u)_X$ represent the fraction of points assigned to each class. The constraint $(u)_X=b$ therefore allows us to incorporate prior knowledge about relative class sizes through the vector $b\in \R^k$, which should have positive entries and sum to one. If there exists $u:X\to S_k$ with $(u)_X=b$, then \eqref{eq:graphcut} admits a solution, which in general may not be unique. 

On its own, the graph-cut problem \eqref{eq:graphcut} can admit many local minimizers that would yield poor classification results. The phenomenon is similar to the degeneracy in Laplace learning at low label rates, since it is very inexpensive to violate any of the label constraints.  Our overall plan is to first use Poisson learning to robustly propagate the labels, and then project onto the constraint set for \eqref{eq:graphcut} and perform several steps of gradient-descent on \eqref{eq:graphcut} to improve the classification accuracy. While Poisson learning propagates the labels in a robust way, the cut energy is more suitable for locating the exact decision boundary. 

To relax the discrete problem \eqref{eq:graphcut}, we approximate the graph-cut energy with the Ginzburg-Landau approximation
\vspace{-2mm}
\begin{equation}\label{eq:GLVar}
\min_{\substack{u\in \lx\\ (u)_X=b}}\Big\{ \GL_\tau(u) -\mu \sum_{j=1}^m (y_j - \cons)\cdot u(x_j)  \Big\},
\end{equation}
\vspace{-2mm}
where
\[ \GL_\tau(u) = \frac{1}{2}\|\nabla u\|_{\lxs}^2 + \frac{1}{\tau}\sum_{i=1}^n\prod_{j=1}^k |u(x_i)-\e_j|^2. \]
The Ginzburg-Landau approximation allows $u\in \lx$ to take on any real values, instead of discrete values $u\in S_k$, making the approximation \eqref{eq:GLVar} easier to solve computationally. The graph Ginzburg-Landau approximation $\GL_\tau$ has been used previously for graph-based semi-supervised learning in \cite{garcia2014multiclass}, and other works have rigorously studied how $\GL_\tau$ approximates graph-cut energies in the scalar setting \cite{van2012gamma}. Here, we extend the results to the vector multi-class setting.
\begin{theorem}\label{thm:GLconv}
Assume $G$ is connected. Let $b\in \R^k$ and assume there exists $u:X\to S_k$ with $(u)_X=b$. For each $\tau>0$ let $u_\tau$ be any solution of \eqref{eq:GLVar}.  Then, the sequence $(u_\tau)_\tau$ is precompact in $\lx$ and any convergent subsequence $u_{\tau_m}$ converges to a solution of the graph-cut problem \eqref{eq:graphcut} as $\tau_m\to 0$. Furthermore, if the solution $u_0:X\to S_k$ of \eqref{eq:graphcut} is unique, then $u_\tau \to u_0$ as $\tau \to 0$.
\end{theorem}

Theorem \ref{thm:GLconv} indicates that we can replace the graph-cut energy \eqref{eq:graphcut} with the simpler Ginzburg-Landau approximation \eqref{eq:GLVar}. To descend on the energy \eqref{eq:GLVar}, we use a time-spitting scheme that alternates gradient descent on 
\[ E_1(u):=\frac{1}{2}\|\nabla u\|_{\lxs}^2 - \mu \sum_{j=1}^m (y_j - \cons)\cdot u(x_j), \]
\vspace{-1mm}
\[\text{and }\ E_2(u):=\frac{1}{\tau}\sum_{i=1}^n\prod_{j=1}^k |u(x_i)-\e_j|^2.\]
The first term $E_1$ is exactly the energy for Poisson learning \eqref{eq:PoissonVar2}, and gradient descent amounts to the iteration
\begin{equation*}
u^{t+1}(x_i)  =  u^t(x_i) - \dd t  \bigg(\L u^t(x_i)  -  \mu  \sum_{j=1}^m (y_j - \cons) \delta_{ij}  \bigg).
\end{equation*}
We note that $\L u$  and the source term above both have zero mean value. Hence, the gradient descent equation for $E_1$ is \emph{volume preserving}, i.e., $(u^{t+1})_X = (u^t)_X$. This would not be true for other fidelity terms, such as an $\ell^2$ fidelity, and this volume conservation property plays an important role in ensuring the class size constraint $(u)_X=b$ in \eqref{eq:GLVar}. 

Gradient descent on the second term $E_2$, when $\tau>0$ is small, amounts to projecting each $u(x_i)\in \R^k$ to the closest label vector $\e_j\in S_k$, while preserving the volume constraint $(u)_X=b$. We approximate this by the following procedure: Let $\proj_{S_k}:\R^k\to S_k$ be the closest point projection, let $s_1,\dots,s_k>0$ be positive weights, and set
\begin{equation}\label{eq:}
u^{t+1}(x_i) = \text{Proj}_{S_k} (\diag(s_1,\dots,s_k) u^t(x_i)),
\end{equation}
where $\diag(s_1,\dots,s_k)$ is the diagonal matrix with diagonal entries $s_1,\dots,s_k$. We use a simple gradient descent scheme to choose the weights $s_1,\dots,s_k>0$ so that the volume constraint $(u^{t+1})_X = b$ holds (see Steps 9-14 in Algorithm \ref{alg:poissonMBO}). By Remark \ref{rem:linearity}, this procedure can be viewed as reweighting the point sources in the Poisson equation \eqref{eq:PoissonL} so that the volume constraint holds. In particular, increasing or decreasing $s_i$ grows or shrinks the size of class $i$.

We note that the work of \cite{jacobs2018auction} provides an alternative way to enforce explicit class balance constraints with a volume constrained MBO method based on auction dynamics. Their method uses a graph-cut based approach with a Voronoi-cell based initialization.

\vspace{-2mm}
\section{Poisson learning algorithms} \label{sec:algorithms}

We now present our proposed Poisson learning algorithms.  The Python source code and simulation environment for reproducing our results is available online.\footnote{\scriptsize Source Code: \url{https://github.com/jwcalder/GraphLearning}}

We let $\bW=(w_{ij})_{i,j=1}^n$ denote our symmetric weight matrix. We treat all vectors as column vectors, and we let $\one$ and $\zero$ denote the all-ones and all-zeros column vectors, respectively, of the appropriate size based on context.  We assume that the first $m$ data points $x_1,x_2,\dots,x_m$ are given labels $y_1,y_2,\dots,y_m\in \{\e_1,\e_2,\dots,\e_k\}$, where the standard basis vector $\e_i\in \R^k$ represents the $i^{\rm th}$ class.  We encode the class labels in a $k\times m$ matrix $\bF$, whose $j^{\rm th}$ column is exactly $y_j$. Let  $\bb\in \R^k$ be the vector whose $i^{\rm th}$ entry $b_i$ is the fraction of data points belonging to class $i$. If this information is not available, we set $\bb=\frac{1}{k}\one$.  

Poisson learning is summarized in Algorithm \ref{alg:poisson}. The label decision for node $i$ is $\ell_i =\argmax_{1\leq j\leq k} \bU_{ij}$, and the reweighting in Step 11 implements the label decision \eqref{eq:modified_labeldec}. In all our results, we always set $\bb=\frac{1}{k}\one$, so Poisson learning does not use prior knowledge of class sizes (the true value for $\bb$ is used in PoissonMBO below). The complexity is $O(TE)$, where $E$ is the number of edges in the graph.  We note that before the reweighting in Step 11, the Poisson learning algorithm computes exactly the function $u_T$ defined in \eqref{eq:PoissonRandomWalk}. In view of this, there is little to be gained from running the iterations beyond the \emph{mixing time} of the random walk. This can be recorded within the loop in Steps 8-10 by adding the iteration $\bp_{t+1} = \bW\bD^{-1} \bp_{t}$, where the initial value $\bp_0$ is the vector with ones in the positions of all labeled vertices, and zeros elsewhere. Up to a constant, $\bp_t$ is the probability distribution of a random walker starting from a random labeled node after $t$ steps. Then the \emph{mixing time} stopping condition is to run the iterations until
\[\|\bp_t - \bp_\infty\|_\infty \leq \eps,\]
where $\bp_\infty=\bW \one /(\one^T \bW \one)$ is the invariant distribution. We use this stopping condition with $\eps=1/n$ in all experiments, which usually takes between 100 and 500 iterations.

\begin{algorithm}[tb]
\caption{PoissonLearning}\label{alg:poisson}
\begin{algorithmic}[1]
\STATE {\bfseries Input:} $\bW,\bF,\bb,T$
\STATE {\bfseries Output:} $\bU \in \R^{n\times k}$
\STATE $\bD \gets \text{diag}(\bW \one)$
\STATE $\bL \gets \bD - \bW$
\STATE $\bcons \gets\frac{1}{m}\bF \one$
\STATE $\bB \gets [\bF - \bcons, \textbf{zeros}(k,n-m)]$
\STATE $\bU \gets \textbf{zeros}(n,k)$
\FOR{$i=1$ {\bfseries to} $T$}
   \STATE $\bU \gets \bU +\bD^{-1}(\bB^T - \bL \bU)$
\ENDFOR
\STATE $ \bU \gets \bU \cdot \diag(\bb/\bcons)$
\end{algorithmic}
\end{algorithm}

The Poisson MBO algorithm is summarized in Algorithm~\ref{alg:poissonMBO}.  The matrices $\bD$, $\bL$ and $\bB$ are the same as in Poisson learning, and Poisson MBO requires an additional fidelity parameter $\mu$ and two parameters $\Ninner$ and $\Nouter$. In all experiments in this paper, we set $\mu=1$, $\Ninner=40$ and $\Nouter=20$. Steps 9-14 implement the volume constrained projection described in Section \ref{sec:poissonMBO}.  We set the time step as $\dd \tau=10$ and set the clipping values in Step 12 to $s_{min}=0.5$ and $s_{max}=2$. We tested on datasets with balanced classes, and on datasets with very unbalanced classes, one may wish to enlarge the interval $[s_{min},s_{max}]$.

The additional complexity of PoissonMBO on top of Poisson learning is $O(\Ninner\Nouter E)$.  On large datasets like MNIST, FashionMNIST and Cifar-10, our Poisson learning implementation in Python takes about 8 seconds to run on a standard laptop computer, and about 1 second with GPU acceleration.\footnote{We used an NVIDIA RTX-2070 GPU, and it took 3 seconds to load data to/from the GPU and 1 second to solve Poisson learning.}  The additional 20 iterations of PoissonMBO takes about 2 minutes on a laptop and 30 seconds on a GPU. These computational times do not include the time taken to construct the weight matrix.

\begin{algorithm}[tb]
\caption{PoissonMBO}\label{alg:poissonMBO}
\begin{algorithmic}[1]
\STATE {\bfseries Input:} $\bW,\bF,\bb,T,\Ninner,\Nouter,\mu>0$
\STATE {\bfseries Output:} $\bU\in \R^{n\times k}$
\STATE $\bU \gets \mu \cdot \text{PoissonLearning}(\bW,\bF,\bb,T)$
\STATE $\dd t \gets 1/\max_{1\leq i\leq n}\bD_{ii}$
\FOR{$i=1$ {\bfseries to} $\Nouter$} 
   \FOR{$j=1$ {\bfseries to} $\Ninner$} 
      \STATE $\bU \gets \bU - \dd t\,(\bL \bU - \mu \bB^T)$
   \ENDFOR
   \STATE $\bs \gets \textbf{ones}(1,k)$ 
   \FOR{$j=1$ {\bfseries to} $100$}
      \STATE $\hat{\bb} \gets \frac{1}{n}\one^T \bproj_{S_k}(\bU \cdot \diag(\bs))$
      \STATE $\bs \gets \textbf{max}(\textbf{min}(\bs + \dd \tau\,(\bb - \hat{\bb}),s_{max}),s_{min})$
   \ENDFOR
   \STATE $\bU \gets \bproj_{S_k}(\bU \cdot \diag(\bs))$
\ENDFOR
\end{algorithmic}
\end{algorithm}

\section{Experimental Results} \label{sec:num}

\begin{table*}[t]
\vspace{-3mm}
\caption{MNIST: Average accuracy scores over 100 trials with standard deviation in brackets.}
\vspace{-3mm}
\label{tab:MNIST}
\vskip 0.15in
\begin{center}
\begin{small}
\begin{sc}
\begin{tabular}{llllll}
\toprule
\# Labels per class&\textbf{1}&\textbf{2}&\textbf{3}&\textbf{4}&\textbf{5}\\
\midrule
Laplace/LP \cite{zhu2003semi}&16.1 (6.2)      &28.2 (10.3)      &42.0 (12.4)      &57.8 (12.3)      &69.5 (12.2)      \\
Nearest Neighbor&55.8 (5.1)      &65.0 (3.2)      &68.9 (3.2)      &72.1 (2.8)      &74.1 (2.4)      \\
Random Walk \cite{zhou2004lazy}&66.4 (5.3)      &76.2 (3.3)      &80.0 (2.7)      &82.8 (2.3)      &84.5 (2.0)      \\
MBO \cite{garcia2014multiclass}&19.4 (6.2)      &29.3 (6.9)      &40.2 (7.4)      &50.7 (6.0)      &59.2 (6.0)      \\
VolumeMBO \cite{jacobs2018auction}&89.9 (7.3)      &95.6 (1.9)      &96.2 (1.2)      &96.6 (0.6)      &96.7 (0.6)      \\
WNLL \cite{shi2017weighted}&55.8 (15.2)      &82.8 (7.6)      &90.5 (3.3)      &93.6 (1.5)      &94.6 (1.1)      \\
Centered Kernel \cite{Mai}&19.1 (1.9)      &24.2 (2.3)      &28.8 (3.4)      &32.6 (4.1)      &35.6 (4.6)      \\
Sparse LP \cite{jung2016semi}&14.0 (5.5)      &14.0 (4.0)      &14.5 (4.0)      &18.0 (5.9)      &16.2 (4.2)      \\
p-Laplace \cite{flores2019algorithms}&72.3 (9.1)      &86.5 (3.9)      &89.7 (1.6)      &90.3 (1.6)      &91.9 (1.0)      \\
{\bf Poisson}       &90.2 (4.0)      &93.6 (1.6)      &94.5 (1.1)      &94.9 (0.8)      &95.3 (0.7)      \\
{\bf PoissonMBO}     &{\bf 96.5 (2.6)}&{\bf 97.2 (0.1)}&{\bf 97.2 (0.1)}&{\bf 97.2 (0.1)}&{\bf 97.2 (0.1)}\\
\bottomrule
\end{tabular}
\end{sc}
\end{small}
\end{center}
\vskip -0.1in
\end{table*}

\begin{table*}[t]
\vspace{-3mm}
\caption{FashionMNIST: Average accuracy scores over 100 trials with standard deviation in brackets.}
\vspace{-3mm}
\label{tab:FashionMNIST}
\vskip 0.15in
\begin{center}
\begin{small}
\begin{sc}
\begin{tabular}{llllll}
\toprule
\# Labels per class&\textbf{1}&\textbf{2}&\textbf{3}&\textbf{4}&\textbf{5}\\
\midrule
Laplace/LP \cite{zhu2003semi}&18.4 (7.3)      &32.5 (8.2)      &44.0 (8.6)      &52.2 (6.2)      &57.9 (6.7)      \\
Nearest Neighbor&44.5 (4.2)      &50.8 (3.5)      &54.6 (3.0)      &56.6 (2.5)      &58.3 (2.4)      \\
Random Walk \cite{zhou2004lazy}&49.0 (4.4)      &55.6 (3.8)      &59.4 (3.0)      &61.6 (2.5)      &63.4 (2.5)      \\
MBO \cite{garcia2014multiclass}&15.7 (4.1)      &20.1 (4.6)      &25.7 (4.9)      &30.7 (4.9)      &34.8 (4.3)      \\
VolumeMBO \cite{jacobs2018auction}&54.7 (5.2)      &61.7 (4.4)      &66.1 (3.3)      &68.5 (2.8)      &70.1 (2.8)      \\
WNLL \cite{shi2017weighted}&44.6 (7.1)      &59.1 (4.7)      &64.7 (3.5)      &67.4 (3.3)      &70.0 (2.8)      \\
Centered Kernel \cite{Mai}&11.8 (0.4)      &13.1 (0.7)      &14.3 (0.8)      &15.2 (0.9)      &16.3 (1.1)      \\
Sparse LP \cite{jung2016semi}&14.1 (3.8)      &16.5 (2.0)      &13.7 (3.3)      &13.8 (3.3)      &16.1 (2.5)      \\
p-Laplace \cite{flores2019algorithms}&54.6 (4.0)      &57.4 (3.8)      &65.4 (2.8)      &68.0 (2.9)      &68.4 (0.5)      \\
{\bf Poisson}        &60.8 (4.6)      &66.1 (3.9)      &69.6 (2.6)      &71.2 (2.2)      &72.4 (2.3)      \\
{\bf PoissonMBO}     &{\bf 62.0 (5.7)}&{\bf 67.2 (4.8)}&{\bf 70.4 (2.9)}&{\bf 72.1 (2.5)}&{\bf 73.1 (2.7)}\\
\bottomrule
\toprule
\# Labels per class&\textbf{10}&\textbf{20}&\textbf{40}&\textbf{80}&\textbf{160}\\
\midrule
Laplace/LP \cite{zhu2003semi}&70.6 (3.1)      &76.5 (1.4)      &79.2 (0.7)      &80.9 (0.5)      &{\bf 82.3 (0.3)}\\
Nearest Neighbor&62.9 (1.7)      &66.9 (1.1)      &70.0 (0.8)      &72.5 (0.6)      &74.7 (0.4)      \\
Random Walk \cite{zhou2004lazy}&68.2 (1.6)      &72.0 (1.0)      &75.0 (0.7)      &77.4 (0.5)      &79.5 (0.3)      \\
MBO \cite{garcia2014multiclass}&52.7 (4.1)      &67.3 (2.0)      &75.7 (1.1)      &79.6 (0.7)      &81.6 (0.4)      \\
VolumeMBO \cite{jacobs2018auction}&74.4 (1.5)      &77.4 (1.0)      &79.5 (0.7)      &{\bf 81.0 (0.5)}&82.1 (0.3)      \\
WNLL \cite{shi2017weighted}&74.4 (1.6)      &77.6 (1.1)      &79.4 (0.6)      &80.6 (0.4)      &81.5 (0.3)      \\
Centered Kernel \cite{Mai}&20.6 (1.5)      &27.8 (2.3)      &37.9 (2.6)      &51.3 (3.3)      &64.3 (2.6)      \\
Sparse LP \cite{jung2016semi}&15.2 (2.5)      &15.9 (2.0)      &14.5 (1.5)      &13.8 (1.4)      &51.9 (2.1)      \\
p-Laplace \cite{flores2019algorithms}&73.0 (0.9)      &76.2 (0.8)      &78.0 (0.3)      &79.7 (0.5)      &80.9 (0.3)      \\
{\bf Poisson}        &75.2 (1.5)      &77.3 (1.1)      &78.8 (0.7)      &79.9 (0.6)      &80.7 (0.5)      \\
{\bf PoissonMBO}     &{\bf 76.1 (1.4)}&{\bf 78.2 (1.1)}&{\bf 79.5 (0.7)}&80.7 (0.6)      &81.6 (0.5)      \\
\bottomrule
\end{tabular}
\end{sc}
\end{small}
\end{center}
\vskip -0.1in
\end{table*}

\begin{table*}[t!]
\vspace{-3mm}
\caption{Cifar-10: Average accuracy scores over 100 trials with standard deviation in brackets.}
\vspace{-3mm}
\label{tab:Cifar10}
\vskip 0.15in
\begin{center}
\begin{small}
\begin{sc}
\begin{tabular}{llllll}
\toprule
\# Labels per class&\textbf{1}&\textbf{2}&\textbf{3}&\textbf{4}&\textbf{5}\\
\midrule
Laplace/LP \cite{zhu2003semi}&10.4 (1.3)      &11.0 (2.1)      &11.6 (2.7)      &12.9 (3.9)      &14.1 (5.0)      \\
Nearest Neighbor&31.4 (4.2)      &35.3 (3.9)      &37.3 (2.8)      &39.0 (2.6)      &40.3 (2.3)      \\
Random Walk \cite{zhou2004lazy}&36.4 (4.9)      &42.0 (4.4)      &45.1 (3.3)      &47.5 (2.9)      &49.0 (2.6)      \\
MBO \cite{garcia2014multiclass}&14.2 (4.1)      &19.3 (5.2)      &24.3 (5.6)      &28.5 (5.6)      &33.5 (5.7)      \\
VolumeMBO \cite{jacobs2018auction}&38.0 (7.2)      &46.4 (7.2)      &50.1 (5.7)      &53.3 (4.4)      &55.3 (3.8)      \\
WNLL \cite{shi2017weighted}&16.6 (5.2)      &26.2 (6.8)      &33.2 (7.0)      &39.0 (6.2)      &44.0 (5.5)      \\
Centered Kernel \cite{Mai}&15.4 (1.6)      &16.9 (2.0)      &18.8 (2.1)      &19.9 (2.0)      &21.7 (2.2)      \\
Sparse LP \cite{jung2016semi}&11.8 (2.4)      &12.3 (2.4)      &11.1 (3.3)      &14.4 (3.5)      &11.0 (2.9)      \\
p-Laplace \cite{flores2019algorithms}&26.0 (6.7)      &35.0 (5.4)      &42.1 (3.1)      &48.1 (2.6)      &49.7 (3.8)      \\
{\bf Poisson}        &40.7 (5.5)      &46.5 (5.1)      &49.9 (3.4)      &52.3 (3.1)      &53.8 (2.6)      \\
{\bf PoissonMBO}     &{\bf 41.8 (6.5)}&{\bf 50.2 (6.0)}&{\bf 53.5 (4.4)}&{\bf 56.5 (3.5)}&{\bf 57.9 (3.2)}\\
\bottomrule
\toprule
\# Labels per class&\textbf{10}&\textbf{20}&\textbf{40}&\textbf{80}&\textbf{160}\\
\midrule
Laplace/LP \cite{zhu2003semi}&21.8 (7.4)      &38.6 (8.2)      &54.8 (4.4)      &62.7 (1.4)      &66.6 (0.7)      \\
Nearest Neighbor&43.3 (1.7)      &46.7 (1.2)      &49.9 (0.8)      &52.9 (0.6)      &55.5 (0.5)      \\
Random Walk \cite{zhou2004lazy}&53.9 (1.6)      &57.9 (1.1)      &61.7 (0.6)      &65.4 (0.5)      &68.0 (0.4)      \\
MBO \cite{garcia2014multiclass}&46.0 (4.0)      &56.7 (1.9)      &62.4 (1.0)      &65.5 (0.8)      &68.2 (0.5)      \\
VolumeMBO \cite{jacobs2018auction}&59.2 (3.2)      &61.8 (2.0)      &63.6 (1.4)      &64.5 (1.3)      &65.8 (0.9)      \\
WNLL \cite{shi2017weighted}&54.0 (2.8)      &60.3 (1.6)      &64.2 (0.7)      &66.6 (0.6)      &68.2 (0.4)      \\
Centered Kernel \cite{Mai}&27.3 (2.1)      &35.4 (1.8)      &44.9 (1.8)      &53.7 (1.9)      &60.1 (1.5)      \\
Sparse LP \cite{jung2016semi}&15.6 (3.1)      &17.4 (3.9)      &20.0 (1.9)      &21.7 (1.3)      &15.0 (1.1)      \\
p-Laplace \cite{flores2019algorithms}&56.4 (1.8)      &60.4 (1.2)      &63.8 (0.6)      &66.3 (0.6)      &68.7 (0.3)      \\
{\bf Poisson}        &58.3 (1.7)      &61.5 (1.3)      &63.8 (0.8)      &65.6 (0.6)      &67.3 (0.4)      \\
{\bf PoissonMBO}     &{\bf 61.8 (2.2)}&{\bf 64.5 (1.6)}&{\bf 66.9 (0.8)}&{\bf 68.7 (0.6)}&{\bf 70.3 (0.4)}\\
\bottomrule
\end{tabular}
\end{sc}
\end{small}
\end{center}
\vskip -0.1in
\end{table*}

\begin{figure*}[t!]
\centering
\subfigure[]{\includegraphics[trim = 10 15 10 10, clip=true,width=0.41\textwidth]{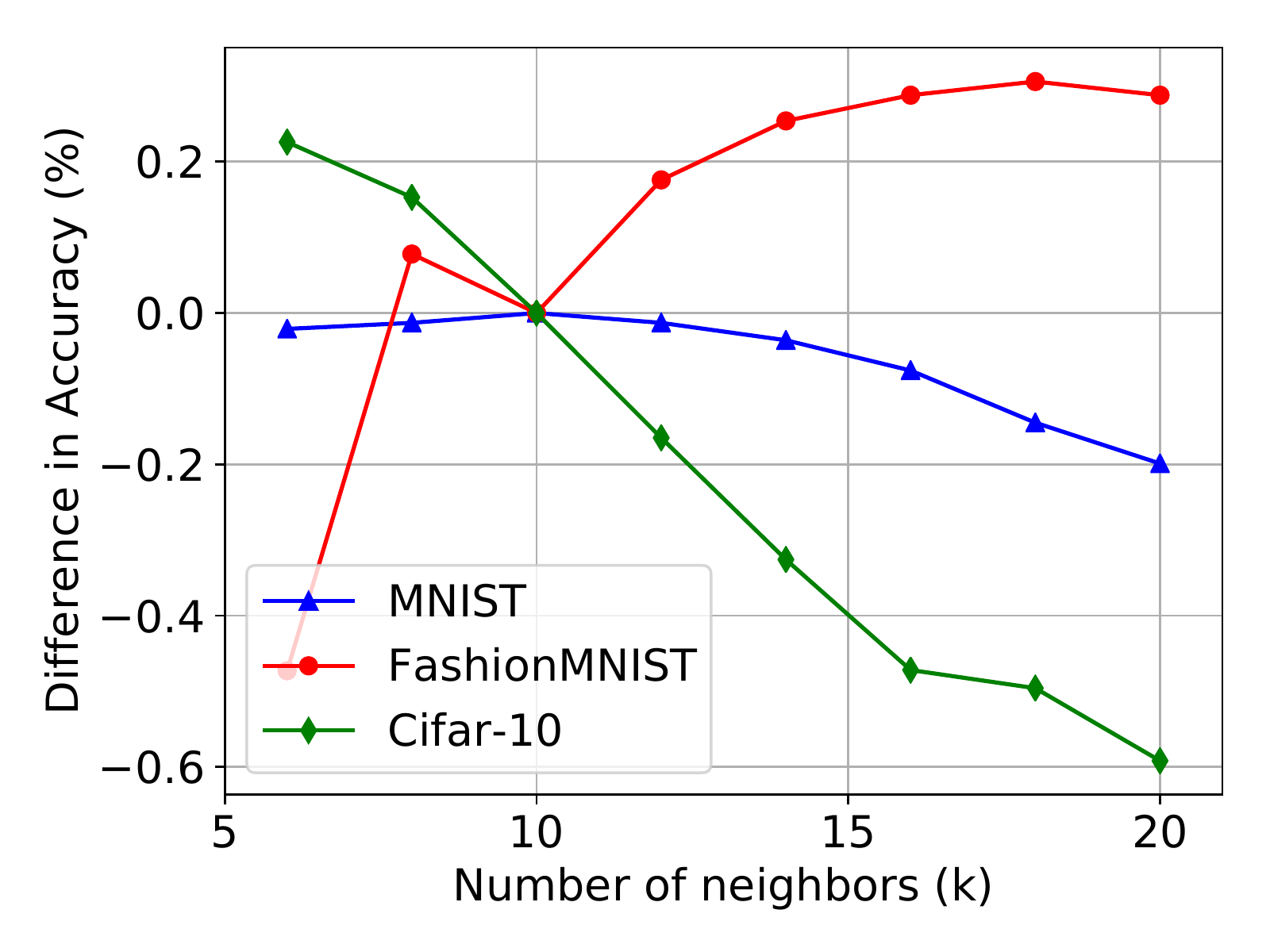}}
\hspace{1cm}
\subfigure[]{\includegraphics[trim = 10 15 10 10, clip=true,width=0.41\textwidth]{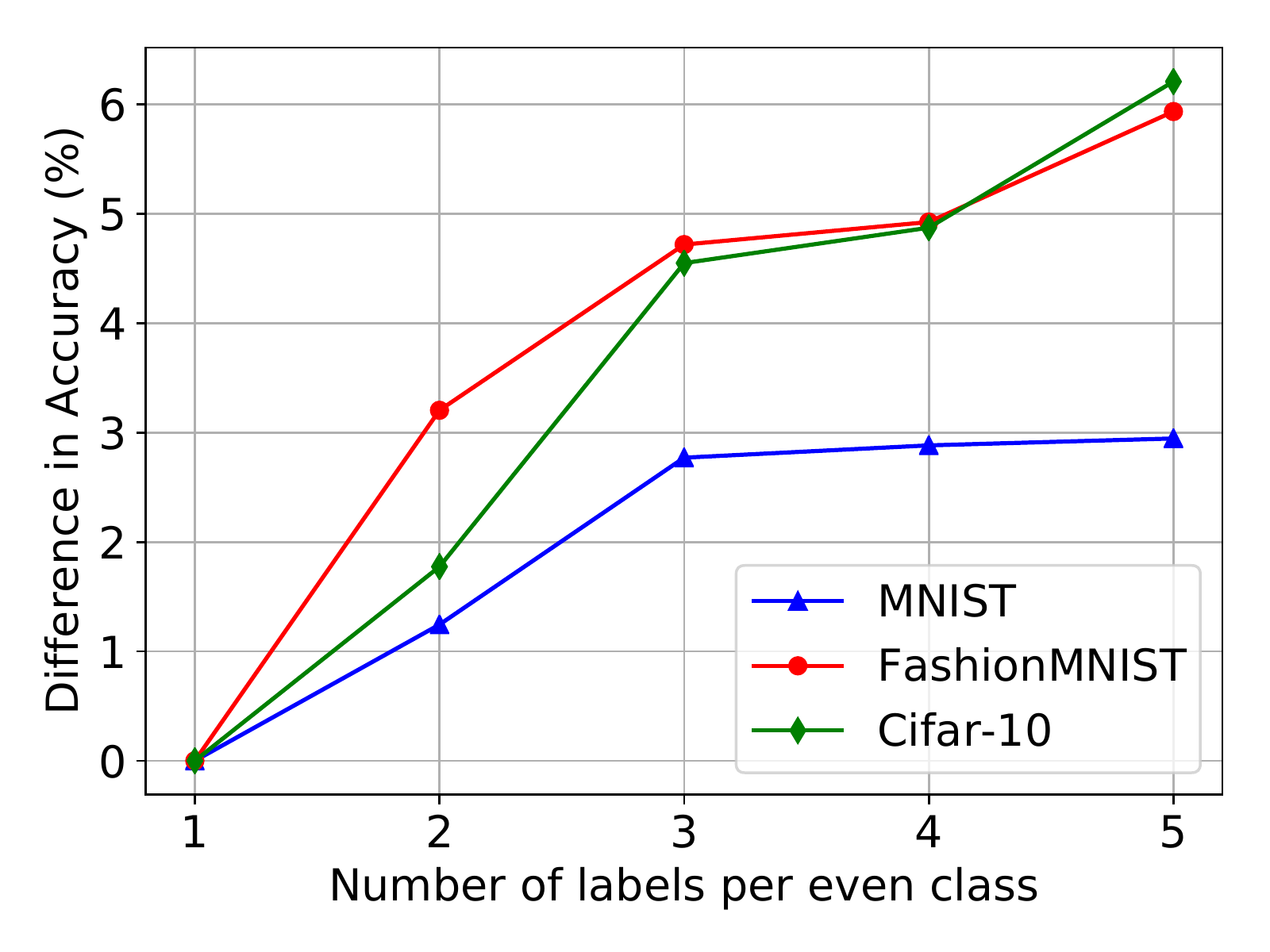}}
\vspace{-4mm}
\caption{Accuracy of Poisson Learning for (a) different numbers of neighbors $k$ used to construct the graph and (b) unbalanced training data. In (a) we used 5 labels per class and in (b) we used 1 label per class for the odd numbered classes, and $m=1,2,3,4,5$ labels per class for the even numbered classes. Both figures show the difference in accuracy compared to $k=10$ and balanced training data. }
\label{fig:kacc}
\end{figure*}

We tested Poisson learning on three datasets: MNIST \cite{lecun1998gradient}, FashionMNIST \cite{xiao2017fashion} and Cifar-10 \cite{krizhevsky2009learning}.  FashionMNIST is a drop-in replacement for MNIST consisting of 10 classes of clothing items. To build good quality graphs, we trained autoencoders to extract important features from the data. For MNIST and FashionMNIST, we used variational autoencoders with 3 fully connected layers of sizes (784,400,20) and (784,400,30), respectively, followed by a symmetrically defined decoder. The autoencoder was trained for 100 epochs on each dataset. The autoencoder architecture, loss, and training, are similar to \cite{kingma2013bayes}. For Cifar-10, we used the AutoEncodingTransformations architecture from \cite{zhang2019aet}, with all the default parameters from their paper, and we normalized the features to unit-vectors.

For each dataset, we constructed a graph over the latent feature space. We used all available data to construct the graph, giving $n=70,000$ nodes for MNIST and FashionMNIST, and $n=60,000$ nodes for Cifar-10. The graph was constructed as a $K$-nearest neighbor graph with Gaussian weights given by
\[w_{ij} =\exp\left( -4|x_i-x_j|^2/d_K(x_i)^2 \right),\]
where $x_i$ represents the latent variables for image $i$, and $d_K(x_i)$ is the distance in the latent space between $x_i$ and its $K^{\rm th}$ nearest neighbor. We used $K=10$ in all experiments. The weight matrix was then symmetrized by replacing $W$ with $W+W^T$. For Poisson learning, we additionally set $w_{ii}=0$ for all $i$. Placing zeros on the diagonal does not change the solution the Poisson learning equation \eqref{eq:PoissonL}, but it does accelerate convergence of the iteration in Algorithm \ref{alg:poisson} by allowing the random walk to propagate faster.

We compare against Laplace learning \eqref{eq:Laplace} \cite{zhu2003semi}, lazy random walks \cite{zhou2004lazy,zhou2004learning}, multiclass MBO \cite{garcia2014multiclass,bertozzi2012diffuse}, weighted nonlocal Laplacian (WNLL) \cite{shi2017weighted}, volume constrained MBO \cite{jacobs2018auction}, Centered Kernel Method \cite{Mai}, sparse label propagation \cite{jung2016semi}, and $p$-Laplace  learning \cite{flores2019algorithms}. In the volume constrained MBO method we used exact volume constraints and temperature of $T=0.1$. In the Centered Kernel Method, we chose $\alpha$ to be $5$\% larger than the spectral norm of the centered weight matrix. For a baseline reference, we also compared against a nearest neighbor classifier that chooses the label of the closest labeled vertex with respect to the graph geodesic distance. In all experiments, we ran 100 trials randomly choosing which data points are labeled, with the exception of the $p$-Laplace and sparse label propagation methods, which are slower and were run for 10 trials. The same random label permutations were used for all methods. 

Tables \ref{tab:MNIST}, \ref{tab:FashionMNIST} and \ref{tab:Cifar10} show the average accuracy and standard deviation over all 100 trials for various low label rates. We also ran experiments at higher label rates on FashionMNIST and Cifar-10, which are reported in the lower half of their respective tables. We mention that in Tables \ref{tab:MNIST}, \ref{tab:FashionMNIST} and \ref{tab:Cifar10} the training data is balanced, so $\bar{y}=\tfrac{1}{10}\one$. Thus, the label decisions \eqref{eq:modified_labeldec} and \eqref{eq:labeldec} are equivalent.

We see that in nearly all cases, PoissonMBO outperforms all other methods, with PoissonMBO typically outperforming Poisson learning by a few percentage points. The most drastic improvements are seen at the ultra low label rates, and at the moderate label rates shown in Tables \ref{tab:FashionMNIST} and \ref{tab:Cifar10}, several other methods perform well. We note that VolumeMBO and PoissonMBO are the only methods that incorporate prior knowledge of class sizes, and are most suitable for direct comparison. Our results can be compared to the clustering results of 67.2\% on FashionMNIST \cite{mcconville2019n2d} and 41.2\% on Cifar-10 \cite{ghasedi2019balanced}.

Figure \ref{fig:kacc}(a) shows the accuracy of Poisson learning at 5 labels per class as a function of the number of neighbors $K$ used in constructing the graph, showing that the algorithm is not particularly sensitive to this.  Figure \ref{fig:kacc}(b) shows the accuracy of Poisson learning for unbalanced training data. We take 1 label per class for half the classes and $m=1,2,3,4,5$ labels per class for the other half. Since the training data is unbalanced, $\bar{y}$ is not a constant vector and the label decision in Step 11 of Algorithm \ref{alg:poisson} (Poisson Learning) compensates for unbalanced training data. Note that in Figure \ref{fig:kacc} we plot the difference in accuracy compared to (a) the baseline of $k=10$ and (b) 1 label per class. In Figure \ref{fig:kacc} (b), we see an increase in accuracy when only half the classes get additional labels, though the increase is not as large as in Tables \ref{tab:MNIST}, \ref{tab:FashionMNIST} and \ref{tab:Cifar10} where all classes get additional labels.

\vspace{-3mm}
\section{Conclusion} \label{sec:conc}

We proposed a new framework for graph-based semi-supervised learning at very low label rates called \emph{Poisson learning}. The method is efficient and simple to implement.  We performed a detailed analysis of Poisson learning, giving random walk and variational interpretations. We also proposed a graph-cut enhancement of Poisson learning, called Poisson MBO, that can give further improvements. We presented numerical results showing that Poisson Learning outperforms all other methods for semi-supervised learning at low label rates on several common datasets.

\section*{Acknowledgements} 

Calder was supported by NSF-DMS grants 1713691,1944925, Cook was supported by a University of Minnesota Grant In Aid award, Thorpe was supported by the European Research Council under the European Union's Horizon 2020 research and innovation programme grant No 777826 (NoMADS) and 647812, and Slep\v{c}ev was supported by NSF-DMS grant 1814991. The authors thank Matt Jacobs for providing code for volume constrained MBO. The authors also thank the Center for Nonlinear Analysis (CNA) at CMU for its support.

%Jeff: Without vfill, the references are all screwed up with huge spaces between them
%\vfill

%\bibliography{ref}
%\bibliographystyle{icml2020}

\onecolumn

\appendix
\appendixpage

\section{Proofs} \label{proofs}

We provide the proofs by section.

\subsection{Proofs for Section \ref{sec:randomwalk}}\label{sec:rwproofs}

We recall  $X_0^x,X_1^x,X_2^x,\dots$ is a random walk on $X$ starting at $X_0^x =x$ with transition probabilities 
%\begin{equation}\label{eq:transitionprob}
\[ \P(X^x_k = x_j \, | \, X^x_{k-1} = x_i) = \frac{w_{ij}}{d_i}. \]
%\end{equation}
Before giving the proof of Theorem \ref{thm:PoissonRandomWalk}, we recall some properties of random walks and Markov chains. The random walk described above induces a Markov chain with state space $X$.   Since the graph is connected and $X$ is finite, the Markov chain is \emph{positive recurrent}. We also assume the Markov chain is \emph{aperiodic}. This implies the distribution of the random walker converges to the invariant distribution of the Markov chain as $k\to \infty$. In particular, choose any initial distribution $p_0\in \lx$ such that $\sum_{i=1}^n p_0(x_i)=1$ and $p_0\geq 0$, and define
\begin{equation}\label{eq:advance}
p_{k+1}(x_i) = \sum_{j=1}^n \frac{w_{ij}}{d_j}p_k(x_j).
\end{equation}
Then $p_k$ is the distribution of the random walker after $k$ steps. Since the Markov chain is positive recurrent and aperiodic we have that
\[\lim_{k\to \infty} p_k(x_i) = \pi(x_i)\]
for all $i$, where 
\[\pi(x_i) = \frac{d_i}{\sum_{i=1}^n d_i}\]
is the invariant distribution of the Markov chain. It is simple to check that if $p_0\in \lx$ is any function (i.e., not necesarily a probability distribution), and we define $p_k$ by the iteration \eqref{eq:advance}, then
\begin{equation}\label{eq:sum}
\lim_{k\to \infty} p_k(x_i) = \pi(x_i)\sum_{j=1}^np_0(x_j).
\end{equation}

We now give the proof of Theorem \ref{thm:PoissonRandomWalk}.

\begin{proof}[Proof of Theorem \ref{thm:PoissonRandomWalk}]
Define the normalized Green's function
\[G_T(x_i,x_j) = \frac{1}{d_i}\E\left[ \sum_{k=0}^T \one_{\{X_k^{x_j}=x_i\}} \right]=\frac{1}{d_i}\sum_{k=0}^T \P(X^{x_j}_k=x_i).\]
Then we have
\begin{align*}
d_iG_T(x_i,x_j) &= \delta_{ij} + \sum_{k=1}^T \sum_{\ell=1}^n\frac{w_{\ell i}}{d_\ell}\P(X^{x_j}_{k-1}=x_\ell)\\
&=\delta_{ij} + \sum_{\ell=1}^n \frac{w_{\ell i}}{d_\ell}\sum_{k=1}^T \P(X^{x_j}_{k-1}=x_\ell)\\
&=\delta_{ij} + \sum_{\ell=1}^n \frac{w_{\ell i}}{d_\ell}\sum_{k=0}^{T-1}  \P(X^{x_j}_{k}=x_\ell)\\
&=\delta_{ij} + \sum_{\ell=1}^n w_{\ell i}G_{T-1}(x_\ell,x_j).
\end{align*}
Therefore we have
\[d_i(G_T(x_i,x_j) - G_{T-1}(x_i,x_j)) + \L G_{T-1}(x_i,x_j) = \delta_{ij},\]
where the Laplacian $\L$ is applied to the first variable of $G_{T-1}$ while the second variable is fixed (i.e. $\L G_{T-1}(x_i,x_j) = [\L G_{T-1}(\cdot,x_j)]_{x_i}$).
Since 
\[u_T(x_i) = \sum_{j=1}^m (y_j-\cons)G_T(x_i,x_j)\]
we have
%\begin{equation}\label{eq:uT}
\[ d_i(u_T(x_i) - u_{T-1}(x_i)) + \L u_{T-1}(x_i) = \sum_{j=1}^m (y_j-\cons)\delta_{ij}. \]
%\end{equation}
Summing both sides over $i=1,\dots,n$ we find that
\[(u_T)_{d,X}=\sum_{i=1}^n d_iu_T(x_i) = \sum_{i=1}^n d_iu_{T-1}(x_i)=(u_{T-1})_{d,X},\]
where $d=(d_1,d_2,\dots,d_n)$ is the vector of degrees.
Therefore $(u_T)_{d,X}=(u_{T-1})_{d,X}=\cdots = (u_0)_{d,X}$. Noting that
\[d_iu_0(x_i) =\sum_{j=1}^m(y_j-\cons)\delta_{ij},\]
we have $(u_0)_{d,X}=0$, and so $(u_T)_{d,X}=0$ for all $T\geq 0$. Let $u\in \lx$ be the solution of
\[\L u(x_i) = \sum_{j=1}^m (y_j-\cons)\delta_{ij}\]
satisfying $(u)_{d,X}=0$. Define $v_T(x_i) = d_i(u_T(x_i) - u(x_i))$. We then check that $v_T$ satisfies 
%\begin{equation}\label{eq:vT}
\[ v_T(x_i) = \sum_{j=1}^n \frac{w_{ij}}{d_j}v_{T-1}(x_j), \]
%\end{equation}
and $(v_T)_X = 0$. Since the random walk is aperiodic and the graph is connected, we have by \eqref{eq:sum} that $\lim_{T\to \infty} v_T(x_i) = \pi(x_i)(v_0)_X=0$, which completes the proof.
\end{proof}

\subsection{Proofs for Section~\ref{sec:variational}}

We first review some additional calculus on graphs. The \emph{graph divergence} of a vector field $V$ is defined as
\[ \div V(x_i) = \sum_{j=1}^nw_{ij}V(x_i,x_j). \]
The divergence is the negative adjoint of the gradient; that is, for every vector field $V\in \lxs$ and function $u\in \lx$ 
\begin{equation}\label{eq:divadjoint}
(\nabla u,V)_{\lxs} = -(u,\div V)_{\lx}.
\end{equation}
We also define  $\|u\|^p_{\lpx} = \sum_{i=1}^n|u(x_i)|^p$ and
\[ \|V\|^p_{\lpxs} = \frac12 \sum_{i,j=1}^nw_{ij}|V(x_i,x_j)|^p, \]
where $|\cdot|$ is the Euclidean norm on $\R^k$.

The graph Laplacian $\L u$ of a function $u\in \lx$ is defined as negative of the composition of gradient and divergence
\[ \L u(x_i) = -\div(\nabla u)(x_i) = \sum_{j=1}^nw_{ij}(u(x_i)-u(x_j)). \]
The operator $\L$ is the \emph{unnormalized} graph Laplacian. Using \eqref{eq:divadjoint} we have
\[ (\L u,v)_{\lx} = (-\div \nabla u,v)_{\lx} = (\nabla u,\nabla v)_{\lxs}. \]
In particular $(\L u,v)_{\lx} = (u,\L v)_{\lx}$, and so the graph Laplacian $\L$ is self-adjoint as an operator $\L:\lx\to \lx$. We also note that 
\[ (\L u,u)_{\lx} = (\nabla u,\nabla u)_{\lxs} = \|\nabla u\|^2_{\lxs}, \]
that is, $\L$ is positive semi-definite.

The variational interpretation of Poisson learning can be directly extended to  $\ell^p$ versions, so we proceed in generality here. For a function $u:X\to \R^k$ and a positive vector $a\in\bbR^n$ (meaning $a_i> 0$ for all $i=1,\dots,n$) we define the weighted mean value
\[ (u)_{a,X} := \frac{1}{\sum_{i=1}^n a_i}\sum_{i=1}^n a_i u(x_i). \]
We define the space of weighted mean-zero functions
\[ \lpxaz = \{u \in \lpx \, : \, (u)_{a,X}=0\}. \]

For $p\geq 1$ and $\mu>0$ we consider the variational problem
\begin{equation}\label{eq:PoissonVar}
\min_{u\in \lpxaz} \hspace{-3pt} \bigg\{ \frac{1}{p}\|\nabla u\|_{\lpxs}^p \hspace{-1pt} - \hspace{-1pt} \mu \sum_{j=1}^m (y_j - \cons)\hspace{-1.5pt}\cdot\hspace{-1.5pt}u(x_j) \bigg\} \hspace{-3pt}
\end{equation}
where $\cons = \frac{1}{m}\sum_{j=1}^m y_j$.  This generalizes the variational problem \eqref{eq:PoissonVar2} for Poisson learning, and the theorem below generalizes Theorem \ref{thm:exist2}.
\begin{theorem}\label{thm:exist}
Assume $G$ is connected. For any $p>1$, positive $a\in\bbR^n$, and $\mu\geq 0$, there exists a unique solution $u\in \lpxaz$ of \eqref{eq:PoissonVar}. Furthermore, the minimizer $u$ satisfies the graph $p$-Laplace equation
\begin{equation}\label{eq:pPoisson}
-\div(|\nabla u|^{p-2}\nabla u)(x_i) =  \mu\sum_{j=1}^m(y_j-\cons)\delta_{ij}.
\end{equation}
\end{theorem}
We give the proof of Theorem \ref{thm:exist} below, after some remarks and other results.
\begin{remark}
When $p=1$, solutions of \eqref{eq:PoissonVar} may not exist for all $\mu\geq 0$, since the variational problem~\eqref{eq:PoissonVar} may not be bounded from below. We can show that there exists $C>0$ such that if $\mu < C$, the variational problem is bounded from below and our argument for existence in Theorem~\ref{thm:exist} goes through. 
\label{rem:p1}
\end{remark}

It turns out that $\mu>0$ is a redundant parameter when $p>1$.
\begin{lemma}\label{lem:mu}
Let $p>1$ and for $\mu>0$ let $u_\mu$ be the solution of \eqref{eq:PoissonVar}. Then, $u_{\mu} = \mu^{1/(p-1)}u_1$.
\end{lemma}
It follows from Lemma \ref{lem:mu} that when $p>1$, the fidelity parameter $\mu>0$ is \emph{completely irrelevant} for classification problems, since the identity $u_{\mu} = \mu^{1/(p-1)}u_1$ implies that the label decision \eqref{eq:labeldec} gives the same labeling for \emph{any} value of $\mu>0$. Hence, in Poisson learning with $p>1$ we always take $\mu=1$. This remark is false for $p=1$.

Before proving Theorem~\ref{thm:exist} we first record a Poincar\'e inequality. The proof is standard but we include it for completeness.
We can prove the Poincar\'e inequality for non-negative vectors $a\in\bbR^n$, meaning that $a_i\geq 0$ for every $i=1,\dots, n$ as long as $\sum_{i=1}^n a_i>0$.

\begin{proposition}\label{prop:poincare}
Assume $G$ is connected, $a\in\bbR^d$ is non-negative with $\sum_{i=1}^n a_i>0$, and $p\geq 1$. There exists $\lambda_p>0$ such that
\begin{equation}\label{eq:Poincare}
\lambda_p \|u - (u)_{a,X}\|_{\lpx} \leq \|\nabla u\|_{\lpxs},
\end{equation}
for all $u\in \lpx$.
\end{proposition}

\begin{proof} %[Proof of Proposition \ref{prop:poincare}]
Define
\[\lambda_p = \min_{\substack{u\in \lpx \\ u\not \equiv (u)_{a,X}}}\frac{\|\nabla u\|_{\lpxs}}{ \|u-(u)_{a,X}\|_{\lpx}}.\]
Then clearly \eqref{eq:Poincare} holds for this choice of $\lambda_p$, and $\lambda_p\geq 0$. If $\lambda_p=0$, then there exists a sequence $u_k\in \lpx$ with $u_k\not\equiv (u_k)_{a,X}$ such that
\[\frac{\|\nabla u_k\|_{\lpxs}}{ \|u_k - (u)_{a,X}\|_{\lpx}}\leq \frac{1}{k}.\]
We may assume that $(u_k)_{a,X}=0$ and $\|u_k\|_{\lpx}=1$, and so 
\begin{equation}\label{eq:gradbound}
\|\nabla u_k\|_{\lpxs}\leq \frac{1}{k}.
\end{equation}
Since $|u_k(x)| \leq \|u_k\|_{\lpx}=1$, the sequence $u_k$ is uniformly bounded and by the Bolzano-Weierstrauss Theorem there exists a subsequence $u_{k_j}$ such that $u_{k_j}(x_i)$ is a convergent sequence in $\R^k$ for every $i$. We denote $u(x_i)=\lim_{j\to\infty}u_{k_j}(x_i)$. Since $\|u_{k_j}\|_{\lpx}=1$ we have $\|u\|_{\lpx}=1$, and thus $u\not \equiv 0$. Similarly, since $(u_k)_{a,X}=0$ we have $(u)_{a,X}=0$ as well.  On the other hand it follows from \eqref{eq:gradbound} that $\|\nabla u\|_{\lpxs}=0$, and so 
\[w_{ij}(u(x_i)-u(x_j)) = 0 \ \ \ \text{ for all }i,j.\]
It follows that $u(x_i)=u(x_j)$ whenever $w_{ij}>0$. Since the graph is connected, it follows that $u$ is constant. Since $(u)_{a,X}=0$ we must have $u\equiv 0$, which is a contradiction, since $\|u\|_{\lpx}=1$. Therefore $\lambda_p>0$, which completes the proof.
\end{proof}

We can now prove Theorem~\ref{thm:exist}.

\begin{proof}[Proof of Theorem~\ref{thm:exist}]
Let us write
\begin{equation}\label{eq:Ip}
I_p(u) =\frac{1}{p} \|\nabla u\|_{\lpxs}^p - \mu \sum_{j=1}^m (y_j - \cons)\cdot u(x_j).
\end{equation}
By Proposition \ref{prop:poincare} we have
\[I_p(u) \geq \frac{1}{p}\lambda_p^p \|u\|_{\lpx}^p - \mu \sum_{j=1}^m (y_j - \cons)\cdot u(x_j)\]
for $u\in \lpxaz$. By H\"older's inequality we have
\begin{align*}
\sum_{j=1}^m (y_j - \cons)\cdot u(x_j)&\leq \sum_{j=1}^m |y_j-\cons||u(x_j)|\\
&\leq \left(\sum_{j=1}^m |y_j-\cons|^q\right)^{1/q}\left(\sum_{j=1}^m |u(x_j)|^p\right)^{1/p}\\
&\leq \left(\sum_{j=1}^m |y_j-\cons|^q\right)^{1/q}\|u\|_{\lpx},
\end{align*}
where $q = p/(p-1)$. Letting $C = \left(\sum_{j=1}^m |y_j-\cons|^q\right)^{1/q}$ we have
\begin{equation}\label{eq:Ipcoercive}
I_p(u) \geq \frac{1}{p}\lambda_p^p \|u\|_{\lpx}^p - C\mu \|u\|_{\lpx}.
\end{equation}
Since $p>1$, we see that $I_p$ is bounded below.

Let $u_k\in \lpxaz$ be a minimizing sequence, that is, we take $u_k$ so that
\[-\infty < \inf_{u\in \lpxaz}I_p(u) = \lim_{k\to \infty}I_p(u_k).\]
By \eqref{eq:Ipcoercive} we have that
\[\frac{1}{p}\lambda_p^p \|u_k\|_{\lpx}^p - C\mu \|u_k\|_{\lpx} \leq \inf_{u\in \lpxaz}I_p(u)+1,\]
for $k$ sufficiently large. Since $p>1$, it follows that there exists $M>0$ such that $\|u_k\|_{\lpx}\leq M$ for all $k\geq 1$. Since $|u_k(x_i)|\leq\|u_k\|_{\lpx}\leq M$ for all $i=1,\dots,n$, we can apply the Bolzano-Weierstrauss Theorem to extract a subsequence $u_{k_j}$ such that $u_{k_j}(x_i)$ is a convergent sequence in $\R^k$ for all $i=1,\dots,n$. We denote by $u^*(x_i)$ the limit of $u_{k_j}(x_i)$ for all $i$. By continuity of $I_p$ we have 
\[\inf_{u\in \lpxaz}I_p(u) = \lim_{j\to \infty}I_p(u_{k_j}) = I_p(u^*),\]
and $(u^*)_{a,X}=0$. This shows that there exists a solution of \eqref{eq:PoissonVar}. 

We now show that any solution of \eqref{eq:PoissonVar} satisfies $-\div\lp|\nabla u|^{p-2}\nabla u\rp = \mu f$. The proof follows from taking a variation. Let $v\in \lpxaz$ and consider $g(t):=I_p(u + tv)$, where $I_p$ is defined in \eqref{eq:Ip}. Then $g$ has a minimum at $t=0$ and hence $g'(0)=0$. We now compute
\begin{align*}
g'(0)&=\frac{d}{dt}\Big\vert_{t=0}\left\{ \frac{1}{p}\|\nabla u + t\nabla v\|_{\lpxs}^p -\mu \sum_{j=1}^m(y_j-\cons)\cdot (u(x_j)+tv(x_j))\right\}\\
&=\frac{1}{2p}\sum_{i,j=1}^nw_{ij}\frac{d}{dt}\Big\vert_{t=0}|\nabla u(x_i,x_j) + t\nabla v(x_i,x_j)|^p - \mu\sum_{j=1}^m(y_j-\cons)\cdot v(x_j)\\
&=\frac{1}{2}\sum_{i,j=1}^nw_{ij}|\nabla u(x_i,x_j)|^{p-2}\nabla u(x_i,x_j)\cdot\nabla v(x_i,x_j) -\mu \sum_{j=1}^m(y_j-\cons)\cdot v(x_j)\\
&=(|\nabla u|^{p-2}\nabla u,\nabla v)_{\lxs} - \mu\sum_{j=1}^m(y_j-\cons)\cdot v(x_j)\\
&=(-\div(|\nabla u|^{p-2}\nabla u),v)_{\lx} - \mu\sum_{j=1}^m(y_j-\cons)\cdot v(x_j)\\
&=(-\div(|\nabla u|^{p-2}\nabla u) - \mu f,v)_{\lx},
\end{align*}
where
\[f(x_i) = \sum_{j=1}^m(y_j-\cons)\delta_{ij}.\]
%Since $v\in \lpxz$ is arbitrary,  we see that $\div(|\nabla u|^{p-2}\nabla u) =\mu f$, which completes the proof.
We choose
\[ v(x_i) = \frac{1}{a_i} \lp -\div\lp|\nabla u|^{p-2}\nabla u\rp(x_i) - \mu f(x_i)\rp \]
then
\[ (v)_{a,X} = \sum_{i=1}^n \lp -\div\lp|\nabla u|^{p-2}\nabla u\rp(x_i) - \mu f(x_i)\rp = 0 \]
so $v\in\lpxaz$.
Moreover, for this choice of $v$,
\[ 0 = g'(0) = \sum_{i=1}^n \frac{1}{a_i} \la \div\lp|\nabla u|^{p-2}\nabla u\rp(x_i) + \mu f(x_i) \ra^2 \geq \frac{1}{\max a_i} \lda \div\lp|\nabla u|^{p-2}\nabla u\rp(x_i) + \mu f(x_i) \rda^2_{\lx}. \]
So, $-\div\lp|\nabla u|^{p-2}\nabla u\rp = \mu f$ as required.

To prove uniqueness, let $u,v\in \lpxaz$ be minimizers of \eqref{eq:PoissonVar}. Then $u$ and $v$ satisfy~\eqref{eq:pPoisson} which we write as
\[ -\div(|\nabla u|^{p-2}\nabla u) =  \mu f. \]
%with $f$ given by
%\[ f(x_i) = \mu\sum_{j=1}^m(y_j-c)\delta_{ij}. \]
Applying Lemma~\ref{lem:quanterror} (below) we find that $\|u - v\|_{\lpx} =0$ and so $u=v$.
\end{proof}

In the above proof we used a quantitive error estimate which is of interest in its own right. 
The estimate was on equations of the form
%\begin{equation} \label{eq:pLaplaceText}
\[ -\div(|\nabla u|^{p-2}\nabla u) =  f \]
%\end{equation}
when $f\in \lpxz$, where we use the notation:
%However, we cannot have any $f$ on the right hand side.
%In particular, if one sums over $i=1,\dots,n$ on the right hand side we can check $\sum_{i=1}^n \div(|\nabla u|^{p-2}\nabla u)(x_i) = 0$ hence we must also have $\sum_{i=1}^n f(x_i) = 0$.
%This turns out to be a necessary and sufficient condition for~\eqref{eq:pLaplaceText} to have a solution.
%To write our statements precisely we introduce the following notation: 
if $a\in\bbR^n$ is a constant vector (without loss of generality the vector of ones) then we write $(u)_{X} = (u)_{a,X} = \frac{1}{n}\sum_{i=1}^n u(x_i)$ and $\lpxz = \{u\in \lpx\,:\,(u)_X=0\}$.

\begin{lemma}
\label{lem:quanterror}
Let $p>1$, $a\in\bbR^n$ be non-negative, and assume $u,v\in \lpxaz$ satisfy
%\begin{equation}\label{eq:plaplace}
\[ -\div(|\nabla u|^{p-2}\nabla u)(x_i) = f(x_i) \]
%\end{equation}
and
%\begin{equation}\label{eq:plaplace2}
\[ -\div(|\nabla v|^{p-2}\nabla v)(x_i) = g(x_i) \]
%\end{equation}
for all $i=1,\dots,n$, where $f,g \in \lpxz$.
Then,
\[ \|u - v\|_{\lpx} \leq \lb \begin{array}{ll} C\lambda_p^{-q}\|f-g\|_{\lqx}^{1/(p-1)} & \text{if } p\geq 2 \\ C\lambda_p^{-2} \lp \|\nabla u\|_{\lpx} + \|\nabla v\|_{\lpx}\rp^{2-p} \| f-g\|_{\lx} & \text{if } 1<p<2 \end{array} \rd \] 
where $C$ is a constant depending only on $p$ and $q=\frac{p}{p-1}$.
\end{lemma}

\begin{remark}
If $-\div(|\nabla u|^{p-2}\nabla u) = f$ then we can write $\lp |\nabla u|^{p-2}\nabla u, \nabla \varphi\rp_{\lxs} = (f,\varphi)_{\lx}$ for any $\varphi\in \lx$.
Choosing $\varphi = u$ implies $\|\nabla u\|_{\lpxs}^p = (f,u)_{\lx} \leq \|f\|_{\lqx} \|u\|_{\lpx}$ so we could write the bound for $p\in (1,2)$ in the above lemma without $\|\nabla u\|_{\lpx}$ and $\|\nabla v\|_{\lpx}$ on the right hand side.
\end{remark}

\begin{proof}
For $p\geq 2$ we use the identity
\[|a-b|^p \leq C(|a|^{p-2}a - |b|^{p-2}b)\cdot(a-b)\]
for vectors $a,b\in \R^k$ for some constant $C$ depending only on $p$ (which can be found in Lemma 4.4 Chapter I~\cite{dibenedetto93}) to obtain
\begin{align*}
\|\nabla u - \nabla v\|_{\lpxs}^p&=\frac{1}{2}\sum_{i,j=1}^nw_{ij} |\nabla u(x_i,x_j) - \nabla v(x_i,x_j)|^p\\
&\leq C\hspace{-3pt}\sum_{i,j=1}^n \hspace{-0.9pt} w_{ij} \left(|\nabla u(x_i,x_j)|^{p-2}\nabla u(x_i,x_j) - |\nabla v(x_i,x_j)|^{p-2}\nabla v(x_i,x_j)\right)\hspace{-1.15pt}\cdot\hspace{-1.15pt}\left(\nabla u(x_i,x_j)  - \nabla v(x_i,x_j)\right)\\ %2^{p-2}
&=C(|\nabla u|^{p-2}\nabla u - |\nabla v|^{p-2}\nabla v,\nabla (u-v))_{\lxs}\\ %2^{p-1}
&=C(-\div(|\nabla u|^{p-2}\nabla u) + \div(|\nabla v|^{p-2}\nabla v),u-v)_{\lx}\\ %2^{p-1}
&=C(f-g,u-v)_{\lx}\\ %2^{p-1}
&\leq C\|f-g\|_{\lqx}\|u-v\|_{\lpx}, %2^{p-1}
\end{align*}
where in the last line we used H\"older's inequality, $\tfrac{1}{p}+\tfrac{1}{q}=1$, and the value of $C$ may change from line-to-line. By Proposition \ref{prop:poincare} we have
\[\lambda_p^p\|u - v\|_{\lpx}^p \leq\|\nabla u - \nabla v\|_{\lpxs}^p \leq C\|f-g\|_{\lqx}\|u-v\|_{\lpx}.\] %2^{p-1}
Therefore we deduce
%\begin{equation}\label{eq:quanterror}
\[ \|u - v\|_{\lpx}  \leq C\lambda_p^{-q}\|f-g\|_{\lqx}^{1/(p-1)}. \] %2
%\end{equation}
%as required.

Now for $1<p<2$ we follow the proof of Lemma 4.4 in Chapter I~\cite{dibenedetto93} to infer
\begin{align*}
\lp |a|^{p-2}a - |b|^{p-2}b\rp \cdot (a-b) & = \int_0^1 \la sa + (1-s)b\ra^{p-2} |a-b|^2 \, \dd s \\
 & \hspace{1cm} + (p-2) \int_0^1 \la sa + (1-s) b\ra^{p-4} \la \lp sa +(1-s)b\rp \cdot (a-b) \ra^2 \, \dd s
\end{align*}
for any $a,b\in\bbR^k$.
Hence, by the Cauchy Schwarz inequality,
\begin{align*}
\lp |a|^{p-2}a - |b|^{p-2}b\rp \cdot (a-b) & \geq (p-1) \int_0^1 \la sa + (1-s) b\ra^{p-2} |a-b|^2 \, \dd s \\
 & \geq (p-1) |a-b|^2 \int_0^1 \frac{1}{(s|a|+(1-s)|b|)^{2-p}} \, \dd s \\
 & \geq \frac{(p-1) |a-b|^2}{(|a|+|b|)^{2-p}}.
\end{align*}
In the sequel we make use of the inequality
\[ \lp |a|^{p-2}a - |b|^{p-2}b\rp \cdot (a-b) \geq \frac{C |a-b|^2}{(|a|+|b|)^{2-p}}. \]

By H\"older's inequality and the above inequality we have (where again the constant $C$ may change from line-to-line)
\begin{align*}
\lda \nabla u - \nabla v \rda_{\lpxs}^p &=\frac{1}{2}\sum_{i,j=1}^n w_{ij}|\nabla u(x_i,x_j) - \nabla v(x_i,x_j)|^p \\
&\leq \lp \frac12 \sum_{i,j=1}^n \frac{w_{ij}|\nabla u(x_i,x_j) - \nabla v(x_i,x_j)|^2}{(|\nabla u(x_i,x_j)| + |\nabla v(x_i,x_j)|)^{2-p}} \rp^{\frac{p}{2}} \lp \frac12 \sum_{i,j=1}^n w_{ij}\lp |\nabla u(x_i,x_j)| + |\nabla v(x_i,x_j)|\rp^p \rp^{\frac{2-p}{2}} \\
 & \leq C \hspace{-3pt} \lp \hspace{-2pt} \sum_{i,j=1}^n \hspace{-3pt} w_{ij} \hspace{-2pt} \lp |\nabla u(x_i,x_j\hspace{-.5pt})|^{p-2}\nabla u(x_i,x_j\hspace{-.5pt}) \hspace{-2pt} - \hspace{-2pt} |\nabla v(x_i,x_j\hspace{-.5pt})|^{p-2}\nabla v(x_i,x_j\hspace{-.5pt})\rp \hspace{-2pt} \cdot \hspace{-2pt} (\nabla u(x_i,x_j\hspace{-.5pt}) \hspace{-2pt} - \hspace{-2pt} \nabla v(x_i,x_j\hspace{-.5pt})) \hspace{-4pt} \rp^{\hspace{-1.25pt}\frac{p}{2}} \\
 & \hspace{1cm} \times \lp \|\nabla u\|_{\lpxs} + \|\nabla v\|_{\lpxs} \rp^{\frac{(2-p)p}{2}} \\
 & = C \lp |\nabla u|^{p-2} \nabla u - |\nabla v|^{p-2} \nabla v , \nabla (u-v) \rp_{\lxs}^{\frac{p}{2}} \lp \|\nabla u\|_{\lpxs} + \|\nabla v\|_{\lpxs} \rp^{\frac{(2-p)p}{2}} \\
 & = C \lp -\div(|\nabla u|^{p-2} \nabla u) + \div(|\nabla v|^{p-2} \nabla v) , u-v \rp_{\lx}^{\frac{p}{2}} \lp \|\nabla u\|_{\lpxs} + \|\nabla v\|_{\lpxs} \rp^{\frac{(2-p)p}{2}} \\
 & = C \lp f-g , u-v \rp_{\lx}^{\frac{p}{2}} \lp \|\nabla u\|_{\lpxs} + \|\nabla v\|_{\lpxs} \rp^{\frac{(2-p)p}{2}} \\
 & \leq C \lda f - g\rda_{\lx}^{\frac{p}{2}} \lda u-v \rda_{\lx}^{\frac{p}{2}} \lp \|\nabla u\|_{\lpxs} + \|\nabla v\|_{\lpxs} \rp^{\frac{(2-p)p}{2}}.
\end{align*}
Combining the above with Proposition~\ref{prop:poincare} we have
\[ \lambda_p^p \| u-v\|_{\lpx}^{\frac{p}{2}} \leq C \lda f - g\rda_{\lx}^{\frac{p}{2}} \lp \|\nabla u\|_{\lpxs} + \|\nabla v\|_{\lpxs} \rp^{\frac{(2-p)p}{2}} \]
which implies the result.
\end{proof}

The final proof from Section~\ref{sec:variational} is Lemma~\ref{lem:mu}.

\begin{proof}[Proof of Lemma~\ref{lem:mu}]
Let us write
\[I_{p,\mu}(u) = \frac{1}{p}\|\nabla u\|_{\lpxs}^p - \mu \sum_{j=1}^m (y_j - \cons)\cdot u(x_j).\]
We note that
\[I_{p,\mu}(\mu^{1/(p-1)}u) = \mu^{p/(p-1)}I_{p,1}(u).\]
Therefore
\[I_{p,\mu}(u_\mu) = \mu^{p/(p-1)}I_{p,1}(u_\mu \mu^{-1/(p-1)}) \geq \mu^{p/(p-1)}I_{p,1}(u_1).\]
On the other hand
\[\mu^{p/(p-1)}I_{p,1}(u_1) = I_{p,\mu}(\mu^{1/(p-1)}u_1) \geq I_{p,\mu}(u_{\mu})\]
Therefore
\[I_{p,\mu}(\mu^{1/(p-1)}u_1) = I_{p,\mu}(u_{\mu}).\]
By uniqueness in Theorem~\ref{thm:exist} we have $u_{\mu} = \mu^{1/(p-1)}u_1$, which completes the proof.
\end{proof}

\subsection{Proofs for Section~\ref{sec:poissonMBO}}

We now turn our attention to the Ginzburg--Landau approximation of the graph cut problem~\eqref{eq:graphcut}.

\begin{proof}[Proof of Theorem \ref{thm:GLconv}]
Let us redefine $\GL_\tau$ in a more general form,
\[ \GL_\tau(u) = \frac{1}{2}\|\nabla u\|_{\lxs}^2 +
\frac{1}{\tau}\sum_{i=1}^n V(u(x_i)) \]
where $V:\bbR^k\to [0,+\infty)$ is continuous and $V(t)=0$ if and only
if $t \in S_k$.
Of course, the choice of $V(t) = \prod_{j=1}^k |t-\e_j|^2$ satisfies
these assumptions.
We let
\begin{align*}
\cE_\tau(u) & = \lb \begin{array}{ll} \GL_\tau(u) - \mu \sum_{j=1}^m
(y_j-\cons) \cdot u(x_j) & \text{if } (u)_X = b \\ + \infty & \text{else,}
\end{array} \rd \\
\cE_0(u) & = \lb \begin{array}{ll} \frac{1}{2}\|\nabla u\|_{\lxs}^2 - \mu \sum_{j=1}^m (y_j-\cons)
\cdot u(x_j) & \text{if } (u)_X = b \text{ and } u:X\to S_k \\ + \infty
& \text{else.} \end{array} \rd
\end{align*}
The theorem can be restated as showing that minimisers $u_\tau$ of
$\cE_\tau$ contain convergent subsequences, and any convergent
subsequence converges to a minimiser of $\cE_0$.
We divide the proof into two steps, in the first step we show that the
sequence of minimisers $\{u_\tau\}_{\tau>0}$ is precompact, in the
second step we show that any convergent subsequence is converging to a
minimiser of $\cE_0$.

\paragraph{\bf 1. Compactness.}
We first show that any sequence $\{u^\prime_\tau\}_{\tau>0}$ and $M\in
\bbR$ satisfying $\sup_{\tau>0} \cE_\tau(u_\tau^\prime)\leq M$ is
precompact.
By Proposition~\ref{prop:poincare} and the Cauchy–Schwarz inequality
\begin{align*}
M & \geq \frac{\lambda_2^2}{2} \| u^\prime_\tau - b\|_{\lx}^2 +
\underbrace{\frac{1}{\tau} \sum_{i=1}^n V(u_\tau^\prime(x_i))}_{\geq 0}
- \mu \underbrace{\sqrt{\sum_{j=1}^m (y_j-\cons)^2}}_{=:C}
\|u^\prime_\tau\|_{\lx} \\
  & \geq \frac{\lambda_2^2}{2} \| u^\prime_\tau - b\|_{\lx}^2 - C\mu
\|u^\prime_\tau - b\|_{\lx} - C\mu \|b\|_{\lx}.
\end{align*}
Hence,
\[ \|u_\tau^\prime - b\|_{\lx} \leq \frac{C\mu}{\lambda_2^2} \left( 1 +\sqrt{1+\frac{2\lambda_2^2 (M+C\mu\| b\|_{\lx})}{C^2\mu^2}}\right)=:\tilde{C} \]
so $\{\mu_\tau^\prime\}_{\tau>0}$ is bounded in $\lx$ and therefore, by
the Bolzano--Weierstrass Theorem, precompact.

To show that minimisers $\{u_\tau\}_{\tau>0}$ are precompact it is
enough to show that there exists $M\in\bbR$ such that $\sup_{\tau>0}
\cE_\tau(u_\tau)\leq M$.
This follows easily as we take $u\in\lx$ satisfying $\sum_{i=1}^n u(x_i)
= b$ and $u(x_i)\in S_k$ for all $i=1,2,\dots,n$ as a candidate.
We have
\[ \cE_\tau(u_\tau) \leq \cE_\tau(u) = \frac12 \|\nabla u\|_{\lxs}^2 -
\mu\sum_{j=1}^m (y_j-\cons) \cdot u(x_j) =: M. \]
Now we have shown that there exists convergent subsequences we show that
any limit must be a minimiser of $\cE_0$.

\paragraph{\bf 2. Converging Subsequences.}
Let $u_0$ be a cluster point of $\{u_\tau\}_{\tau>0}$, i.e. there exists
a subsequence such that $u_{\tau_m} \to u_0$ as $m\to\infty$.
Pick any $v\in \lx$ with $\cE_0(v)<+\infty$.
We will show
\begin{enumerate}
\item[(a)] $\cE_{\tau}(v) = \cE_0(v)$,
\item[(b)] $\liminf_{\tau\to 0} \cE_\tau(u_\tau) \geq \cE_0(u_0)$.
\end{enumerate}
Assuming (a) and (b) hold then, by (a),
\[ \cE_0(v) = \cE_{\tau_m}(v) \geq \cE_{\tau_m}(u_{\tau_m}). \]
Taking the limit as $m\to\infty$, and applying (b) we have
\[ \cE_0(v) \geq \liminf_{m\to\infty} \cE_{\tau_m}(u_{\tau_m}) \geq
\cE_0(u_0). \]
It follows that for all $v\in \lx$ we have $\cE_0(u_0)\leq \cE_0(v)$,
hence $u_0$ is a minimiser of $\cE_0$.

To show (a), we easily notice that
\[ \cE_{\tau}(v) = \frac12\|\nabla v\|_{\lxs}^2 +
\frac{1}{\tau} \sum_{i=1}^n \underbrace{V(v(x_i))}_{=0} - \mu
\sum_{j=1}^m (y_j-\cons) \cdot v(x_j) = \cE_0(v). \]

For (b) we without loss of generality assume that $u_\tau\to u_0$ and
\[ \liminf_{\tau\to0} \cE_\tau(u_\tau) = \lim_{\tau\to0}
\cE_\tau(u_\tau) <+\infty. \]
As $\sum_{i=1}^n u_{\tau}(x_i) = b$ for all $\tau>0$ and $u_\tau(x_i)
\to u_0(x_i)$ for every $i\in\{1,\dots,n\}$ then $(u_0)_X = \sum_{i=1}^n
u_0(x_i) = b$.
And since $V(u_\tau(x_i))\leq \tau \cE_\tau(u_\tau)\to 0$ then we have
$V(u_0(x_i))=0$, hence $u_0(x_i)\in S_k$.
Now,
\[ \cE_\tau(u_\tau) = \underbrace{\frac12\|\nabla u_\tau\|_{\lxs}^2}_{\to \frac12\|\nabla u_0\|_{\lxs}^2} + \frac{1}{\tau}
\sum_{i=1}^n \underbrace{V(u_\tau(x_i))}_{\geq 0} - \mu
\underbrace{\sum_{j=1}^m (y_j-\cons)\cdot u_\tau(x_j)}_{\to \sum_{j=1}^m
(y_j-\cons)\cdot u_0(x_j)}. \]
So $\liminf_{\tau\to 0} \cE_\tau(u_\tau) \geq \cE_0(u_0)$ as required.
\end{proof}

\begin{remark}
If (a) and (b) in the proof of Theorem~\ref{thm:GLconv} are strengthened
to
\begin{enumerate}
\item[(a$^\prime$)] for all $v\in \lx$ there exists $v_\tau\to v$ such that
$\lim_{\tau\to 0}\cE_{\tau}(v_\tau) = \cE_0(v)$,
\item[(b$^\prime$)] for all $v\in \lx$ and for all $v_\tau\to v$ then
$\liminf_{\tau\to 0} \cE_\tau(v_\tau) \geq \cE_0(v)$
\end{enumerate}
then one says that $\cE_\tau$ $\Gamma$-converges to $\cE_0$ (and one can
show that (a$^\prime$) and (b$^\prime$) hold in our case with a small modification of
the above proof).
The notion of $\Gamma$-convergence is fundamental in the calculus of
variations and is considered the variational form of convergence as it
implies (when combined with a compactness result) the convergence of
minimisers.
\end{remark}

\section{Continuum limits} \label{sec:continuum}

We briefly discuss continuum limits for the Poisson learing problem \eqref{eq:PoissonL}. We take $p=2$ for simplicity, but the arguments extend similarly to other values of $p\geq 1$. In order to analyze continuum limits of graph-based learning algorithms, we make the \emph{manifold assumption}, and assume $G$ is a  random geometric graph sampled from an underlying manifold.  To be precise, we assume the vertices of the graph corresponding to unlabeled points $x_1,\dots,x_n$ are a sequence of \emph{i.i.d.}~random variables drawn from a $d$-dimensional compact, closed, and connected manifold $\M$ embedded in $\R^D$, where $d\ll D$. We assume the probability distribution of the random variables has the form $\dd \mu = \rho \dd \Vol_\M$, where $\Vol_\M$ is the volume form on the manifold, and $\rho$ is a smooth density. For the labeled vertices in the graph, we take a fixed finite set of points $\Gamma \subset \M$. The vertices of the random geometric graph are 
\[X_n:=\{x_1,\dots,x_n\}\cup \Gamma.\]
We define the edge weights in the graph by
%\begin{equation}\label{eq:weights}
\[ w_{xy} = \eta_\eps\left(|x-y|\right), \]
%\end{equation}
where $\eps>0$ is the length scale on which we connect neighbors, $|x-y|$ is Euclidean distance in $\R^D$, and $\eta:[0,\infty)\to [0,\infty)$ is smooth with compact support, and $\eta_\eps(t) = \frac{1}{\eps^d}\eta\left( \tfrac{t}{\eps} \right)$. We denote the solution of the Poisson learning problem \eqref{eq:PoissonL} for this random geometric graph by $u_{n,\eps}(x)$.

The normalized graph Laplacian is given by
\[\L_{n,\eps} u(x) = \frac{2}{\sigma_\eta n\eps^2}\sum_{y\in X_n}\eta_\eps(|x-y|)(u(x) - u(y)),\]
where $\sigma_\eta = \int_{\R^d}|z_1|^2 \eta(|z|)\, dz$. It is well-known (see, e.g., \cite{hein2007graph}), that $\L_{n,\eps}$ is consistent with the (negative of) the weighted Laplace-Beltrami operator
\[\Delta_\rho:=-\rho^{-1}\div_\M(\rho^2 \nabla_\M u),\]
where $\div_\M$ is the manifold divergence and $\nabla_\M$ is the manifold gradient. We write $\div=\div_\M$ and $\nabla= \nabla_\M$ now for convenience. In particular, for any  $u\in C^3(\M)$ we have
\[|\L_{n,\eps} u(x) - \Delta_\rho u(x)| \leq C(\|u\|_{C^3(\M)}+1)(\lambda + \eps)\]
holds for all $x\in X_n$ with probability at least $1-Cn\exp\left( -cn\eps^{d+2}\lambda^2 \right)$ for any $0 < \lambda \leq 1$, where $C,c>0$ are constants.

%In terms of the normalized graph Laplacian, the Poisson learning problem \eqref{eq:PoissonL} becomes
Using the normalised graph Laplacian in the Poisson learning problem~\eqref{eq:PoissonL} we write
\begin{equation} \label{eq:NormPoisson}
\L_{n,\eps} u_{n,\eps}(x) = n\sum_{y\in \Gamma}^m(g(y)-\cons)\delta_{x=y} \ \ \ \text{for }x\in X_n,
\end{equation}
where $g(y)\in \R$ denotes the label associated to $y\in \Gamma$ and $\cons = \frac{1}{|\O|}\sum_{x\in \O}g(x)$.
We restrict to the scalar case (binary classification) for now.
Note that the normalisation plays no role in the classification problem~\eqref{eq:labeldec}.
To see what should happen in the continuum, as $n\to \infty$ and $\eps\to 0$, we multiply both sides of \eqref{eq:NormPoisson} by a smooth test function $\varphi\in C^\infty(\M)$, sum over $x\in X$, and divide by $n$ to obtain
\begin{equation}\label{eq:weakgraph}
\frac{1}{n}(\L_{n,\eps}u_{n,\eps},\phi)_{\lx} = \sum_{y\in \O}(g(y)-\cons)\phi(y). %\frac{2}{\sigma_\eta n^2\eps^{d+2}}
\end{equation}
Since $\L_{n,\eps}$ is self-adjoint (symmetric), we have
\begin{equation*}
(\L_{n,\eps}u_{n,\eps},\phi)_{\lx} = (u_{n,\eps},\L_{n,\eps}\phi)_{\lx} = (u_{n,\eps},\Delta_\rho \phi)_{\lx} + O\left( (\lambda+\eps)\|u_{n,\eps}\|_{\lonex} \right).
\end{equation*}
We also note that
\[\sum_{y\in \O}(g(y)-\cons)\phi(y) = \int_\M \sum_{y\in \O}(g(y)-\cons)\delta_y(x)\phi(x)\, \dd \Vol_\M(x),\]
where $\delta_y$ is Dirac-Delta distribution centered at $y\in \M$, which has the property that
\[\int_\M \delta_y(x)\phi(x) \, \dd\Vol_\M(x) = \phi(y)\]
for every smooth $\phi\in C^\infty(\M)$. Combining these observations with \eqref{eq:weakgraph} we see that
\[ \frac{1}{n}(u_{n,\eps},\Delta_\rho \phi)_{\lx} + O\left( \frac{(\lambda+\eps)}{n}\|u_{n,\eps}\|_{\lonex} \right) = \int_\M \sum_{y\in \O}(g(y)-\cons)\delta_y(x)\phi(x)\,\dd\Vol_\M(x).\] %\frac{2}{\sigma_\eta n^2\eps^{d+2}}
%This suggests the normalization
%\[w_{n,\eps}(x) = \frac{\sigma_\eta}{2}n^2\eps^{d+2}u_{n,\eps}(x).\]
%Then $w_{n,\eps}$ satisfies
%\[\frac{1}{n} (w_{n,\eps},\Delta_\rho \phi)_{\lx} + O\left( \frac{(\lambda+\eps)}{n}\|w_{n,\eps}\|_{\lonex} \right) = \int_\M \sum_{y\in \O}(g(y)-c)\delta_y(x)\phi(x)\,\dd\Vol_\M(x).\]
If we can extend $u_{n,\eps}$ to a function on $\M$ in a suitable way, then the law of large numbers would yield
\[\frac{1}{n} (u_{n,\eps},\Delta_\rho \phi)_{\lx} \approx \int_{\M}u_{n,\eps}(x)\rho(x)\Delta_\rho \varphi(x)\, \dd\Vol_\M(x).\]
Hence, if $u_{n,\eps}\to u$ as $n\to \infty$ and $\eps\to 0$ in a sufficiently strong sense, then the function $u:\M\to \R$ would satisfy
\[-\int_{\M}u\,\div(\rho^2 \nabla \varphi)\, \dd\Vol_\M = \int_\M \sum_{y\in \O}(g(y)-\cons)\delta_y(x)\phi(x)\,\dd\Vol_\M(x)\]
for every smooth $\varphi\in C^\infty(\M)$. If $u\in C^2(\M)$, then we can integrate by parts on the left hand side to find that
\[-\int_{\M}\phi\, \div(\rho^2 \nabla u)\, \dd\Vol_\M = \int_\M \sum_{y\in \O}(g(y)-\cons)\delta_y(x)\phi(x)\,\dd\Vol_\M(x)\]
Since $\phi$ is arbitrary, this would show that $u$ is the solution of the Poisson problem 
\begin{equation}\label{eq:PoissonCont}
 -\div\left( \rho^2 \nabla u\right)= \sum_{y\in \O}(g(y)-\cons)\delta_y \ \ \ \text{ on }\M.
\end{equation}
We conjecture that the solutions $u_{n,\eps}$ converge to the solution of \eqref{eq:PoissonCont} as $n\to \infty$ and $\eps\to 0$ with probability one. 

\begin{conjecture}\label{conj:continuum}
Assume $\rho$ is smooth. Assume that $n\to \infty$ and $\eps=\eps_n\to 0$ so that
%\begin{equation}\label{eq:neps}
\[ \lim_{n\to \infty}\frac{n\eps^{d+2}}{\log n} = \infty. \]
%\end{equation}
Let $u \in C^\infty(\M\setminus \Gamma)$ be the solution of the Poisson equation \eqref{eq:PoissonCont} and $u_{n,\eps}$ solve the graph Poisson problem~\eqref{eq:NormPoisson}. Then with probability one
\[\lim_{n\to \infty}\max_{\substack{x\in X_n\\ \dist(x,\Gamma)>\delta}}\left| u_{n,\eps}(x) - u(x)\right| = 0\]
for all $\delta>0$.
\end{conjecture}
The conjecture states that $u_{n,\eps}$ converges to $u$ uniformly as long as one stays a positive distance away from the source points $\Gamma$, where the solution $u$ is singular. We expect the conjecture to be true, since similar results are known to hold when the source term on the right hand side is a smooth function $f$. The fact that the right hand side in \eqref{eq:PoissonCont} is highly singular, involving delta-mass concentration, raises difficult technical problems that will require new insights that are far beyond the scope of this paper.

\begin{remark}
If Conjecture \ref{conj:continuum} is true, it shows that Poisson learning is consistent with a well-posed continuum PDE for arbitrarily low label rates. This is in stark contrast to Laplace learning, which does not have a well-posed continuum limit unless the number of labels grows to $\infty$ as $n\to \infty$ sufficiently fast. This partially explains the superior performance of Poisson learning for low label rate problems.
\label{rem:conj}
\end{remark}

\end{document}